%% file: paper.tex
\documentclass[conference]{IEEEtran}
\usepackage{cite}
\usepackage{amsmath,amssymb,amsfonts}
\usepackage{graphicx}
\usepackage{textcomp}
\usepackage{xcolor}
\def\BibTeX{{\rm B\kern-.05em{\sc i\kern-.025em b}\kern-.08em
    T\kern-.1667em\lower.7ex\hbox{E}\kern-.125emX}}

\usepackage{xspace}
\usepackage[inline]{enumitem}
\usepackage{paralist}
\usepackage{bbding}
\usepackage{diagbox}
\usepackage{pifont}
\usepackage{algorithm, algpseudocode}
\usepackage{algorithmicx}
\usepackage{multirow}
\usepackage{colortbl}
\usepackage{url}
\usepackage{marginnote}
\usepackage{xstring}
\usepackage{flushend}

\input{dfn}

\begin{document}

\title{\method: Scalable Microcluster Detection in Dimensional and Nondimensional Datasets}

\author{\IEEEauthorblockN{Braulio V. Sánchez Vinces}
\IEEEauthorblockA{
\textit{ICMC, University of São Paulo, Brazil} \\
braulio.sanchez@usp.br
}
\and
\IEEEauthorblockN{Robson L. F. Cordeiro}
\IEEEauthorblockA{
\textit{SCS, Carnegie Mellon University, USA} \\
robsonc@andrew.cmu.edu
}
\and
\IEEEauthorblockN{Christos Faloutsos}
\IEEEauthorblockA{
\textit{SCS, Carnegie Mellon University, USA}\\
christos@cs.cmu.edu
}
}

\maketitle

\begin{abstract}
  \input{000abstract}
\end{abstract}


\vspace{-1mm}
\section{Introduction}
\label{sec:intro}
\input{010introduction}

\vspace{-1mm}
\section{Problem \& Related Work}
\label{sec:probrel}
\vspace{-1.1mm}
\subsection{Problem Statement}
\input{021problem}
\vspace{-1.1mm}
\subsection{Related Work}
\input{022relatedwork}

\section{Proposed Axioms}
\label{sec:axioms}
\input{030axioms}

\vspace{-3mm}
\section{Proposed Method}
\label{sec:meth}
\input{040method}

\vspace{-1.5mm}
\section{Experiments}
\label{sec:exp}
\input{050experiment}

\vspace{-2mm}
\section{Conclusions}
\label{sec:concl}
\input{060conclusion}

\bibliographystyle{IEEEtran}
\bibliography{ref}

\end{document}

%% file: dfn.tex
\newtheorem{proof}{Proof}
\newtheorem{lemma}{Lemma}
\newtheorem{definition}{Definition}
\newtheorem{Problem}{Problem}

\newcolumntype{?}{!{\vrule width 1.5pt}}
\newcommand{\checkmarkcompetitor}{\Large{\Checkmark}}
\newcommand{\checkmarkmethod}{\Large{\CheckmarkBold}}
\newcommand{\questionmarkcompetitor}{\Large{\textbf{?}}}

\newcommand{\repository}{\url{https://figshare.com/s/08869576e75b74c6cc55}\xspace}

\newcommand{\method}{\funcname{McCatch}\xspace}
\newcommand{\numberradiiname}{{Number of Radii}\xspace}
\newcommand{\maxslopename}{{Maximum Plateau Slope}\xspace}

\newcommand{\maxsizename}{{Maximum Microcluster Cardinality}\xspace}
\newcommand{\maxsizenameshort}{{Max. Mc. Cardinality}\xspace}
\newcommand{\mindistname}{{Cutoff}\xspace}
\newcommand{\mindistnamebold}{{\textbf{Cutoff}}\xspace}
\newcommand{\transformationcostname}{{Transformation Cost}\xspace}
\newcommand{\transformationcostnamebold}{{\textbf{Transformation Cost}}\xspace}
\newcommand{\histname}{{Histogram of $1$NN Distances}\xspace}
\newcommand{\histnamebold}{{\textbf{Histogram of $\mathbf{1}$NN Distances}}\xspace}
\newcommand{\costcallname}{{Cost of Compression}\xspace}
\newcommand{\fplengthname}{{$1$NN Distance}\xspace}

\newcommand{\fplengthsname}{{$1$NN Distances}\xspace}
\newcommand{\fplengthnamebold}{\textbf{$\mathbf{1}$NN Distance}\xspace}
\newcommand{\mplengthname}{{Group $1$NN Distance}\xspace}

\newcommand{\mplengthnamebold}{\textbf{Group $\mathbf{1}$NN Distance}\xspace}
\newcommand{\scorename}{{Score}\xspace}

\newcommand{\degreesofisolationname}{`Bridges' Lengths'\xspace}
\newcommand{\degreeofisolationname}{`Bridge's Length'\xspace}

\newcommand{\degreeofisolationnamebold}{\textbf{`Bridge's Length'}\xspace}

\newcommand{\iaxiom}{{Isolation Axiom}\xspace}
\newcommand{\iaxiomshort}{{Isolation Ax.}\xspace}
\newcommand{\caxiom}{{Cardinality Axiom}\xspace}
\newcommand{\caxiomshort}{{Cardinality Ax.}\xspace}

\newcommand{\scalable}{{Scalable\xspace}}
\newcommand{\handsoff}{{`Hands-Off'\xspace}}
\newcommand{\generalinput}{{General Input\xspace}}
\newcommand{\generaloutput}{{General Output\xspace}}
\newcommand{\principled}{{Principled\xspace}}
\newcommand{\accurate}{{Accurate}\xspace}
\newcommand{\practical}{{Practical}\xspace}

\makeatletter
\algnewcommand{\LineComment}[1]{\Statex \hskip\ALG@thistlm \(\textcolor{red}{\triangleright}\) \textcolor{red}{#1}}
\algnewcommand{\LineSecCommentPos}[2]{\Statex \hskip\ALG@thistlm \(\textcolor{blue}{\hspace*{#1}\triangleright}\) \textcolor{blue}{#2}}
\algnewcommand{\LineSecComment}[1]{\LineSecCommentPos{3mm}{#1}}
\algnewcommand{\SideComment}[1]{\textcolor{blue}{\Comment{#1}}}
\algnewcommand{\SideCommentPos}[2]{\textcolor{blue}{\hspace*{#1}\Comment{#2}}}
\makeatother
\newcommand{\Input}{\textbf{Input}\xspace}
\newcommand{\Output}{\textbf{Output}\xspace}

\newcommand{\mult}{\cdot}
\newcommand{\cartprod}{\times}

\newcommand{\suchthat}{:}
\newcommand{\domain}[1]{\mathbb{#1}}
\newcommand{\reals}{\domain{R}}
\newcommand{\naturals}{\domain{N}}
\newcommand{\distfunc}[2]{\mathtt{f}(#1,#2)}
\newcommand{\dataset}{\mathbf{P}}
\newcommand{\numpoints}{n}
\newcommand{\intdim}{u}
\newcommand{\datasubset}[1]{\mathbf{#1}}
\newcommand{\point}[1]{\mathbf{p}_{#1}}
\newcommand{\pointi}{\point{i}}
\newcommand{\pointiprime}{\point{i'}}
\newcommand{\pointiprimeprime}{\point{i''}}
\newcommand{\diamdataset}{l}
\newcommand{\radii}{\mathcal{R}}
\newcommand{\radius}[1]{r_{#1}}
\newcommand{\radiuse}{\radius{e}}
\newcommand{\radiuseprime}{\radius{e'}}

\newcommand{\nn}[2]{q_{#1}^{\left(#2\right)}}
\newcommand{\nnei}{\nn{e}{i}}

\newcommand{\nnin}[2]{f_{#1}^{\left(#2\right)}}
\newcommand{\nninei}{\nnin{e}{i}}

\newcommand{\dni}[1]{g_{#1}}
\newcommand{\dnii}{\dni{i}}

\newcommand{\plateau}[1]{\mathtt{U}^{\left(#1\right)}}
\newcommand{\plateaui}{\plateau{i}}
\newcommand{\fplateau}[1]{\mathtt{X}^{\left(#1\right)}}
\newcommand{\fplateaui}{\fplateau{i}}
\newcommand{\mplateau}[1]{\mathtt{Y}^{\left(#1\right)}}
\newcommand{\mplateaui}{\mplateau{i}}
\newcommand{\outliers}{\datasubset{A}}
\newcommand{\mcs}{\mathcal{M}}
\newcommand{\nummcs}{m}
\newcommand{\mc}[1]{\mathbf{M}_{#1}}
\newcommand{\mcj}{\mc{j}}

\newcommand{\scores}{\mathcal{S}}
\newcommand{\score}[1]{s_{#1}}
\newcommand{\scorej}{\score{j}}

\newcommand{\scoresp}{\mathcal{W}}
\newcommand{\scorep}[1]{w_{#1}}
\newcommand{\scorepi}{\scorep{i}}

\newcommand{\tree}{\mathtt{T}}
\newcommand{\treeprime}{\mathtt{T'}}
\newcommand{\op}{\mathtt{O}}
\newcommand{\graph}{\mathtt{G}}
\newcommand{\nodes}{\datasubset{M}}
\newcommand{\edges}{\mathcal{E}}
\newcommand{\component}[1]{\mathtt{G}_{#1}}
\newcommand{\componentj}{\component{j}}
\newcommand{\nodescomponent}[1]{\datasubset{M}_{#1}}
\newcommand{\nodescomponentj}{\nodescomponent{j}}
\newcommand{\edgescomponent}[1]{\mathcal{E}_{#1}}
\newcommand{\edgescomponentj}{\edgescomponent{j}}
\newcommand{\numberradii}{a}
\newcommand{\numberradiidefinition}{\numberradii \in \naturals_{> 1}}
\newcommand{\numberradiidefaultonlyvalue}{15}
\newcommand{\numberradiidefault}{\numberradii = 15}
\newcommand{\maxslope}{b}
\newcommand{\maxslopedefinition}{\maxslope \in \reals_{\geq 0}}
\newcommand{\maxslopedefaultonlyvalue}{1/10}
\newcommand{\maxslopedefault}{\maxslope = 0.1}
\newcommand{\maxsize}{c}
\newcommand{\maxsizedefinition}{\maxsize \in \naturals_{\geq 1}}
\newcommand{\maxsizedefaultonlyvalue}{1/10}
\newcommand{\maxsizedefault}{\maxsize = \left \lceil{\numpoints \mult 0.1}\right \rceil}
\newcommand{\mindist}{d}
\newcommand{\funcname}[1]{\textsc{#1}}
\newcommand{\funcdef}[1]{\funcname{#1}\left( \ \ \right)}
\newcommand{\methoddef}{\funcdef{\method}}
\newcommand{\buildopdef}{\funcdef{BuildOPlot}}
\newcommand{\spotmcsdef}{\funcdef{SpotMCs}}
\newcommand{\scoremcsdef}{\funcdef{ScoreMCs}}

\newcommand{\funccall}[3]{#1 = \funcname{#2}\left(#3\right)}
\newcommand{\funccallnoreturn}[2]{\funcname{#1}\left(#2\right)}
\newcommand{\buildopcall}{\funccall{\op}{BuildOPlop}{\dataset, \tree, \radii, \maxslope, \maxsize}}
\newcommand{\spotmcscall}{\funccall{\mcs}{SpotMCs}{\dataset, \op, \radii}}
\newcommand{\scoremcscall}{\funccall{\scores,\scoresp}{ScoreMCs}{\dataset,\mcs,\op,\radii}}
\newcommand{\joinccall}[2]{\funccall{#1}{JoinC}{#2}}

\newcommand{\selfjoinccall}[2]{\funccall{#1}{SelfJoinC}{#2}}
\newcommand{\selfjoinccallnoreturn}[1]{\funccallnoreturn{SelfJoinC}{#1}}
\newcommand{\selfjoincall}[2]{\funccall{#1}{SelfJoin}{#2}}
\newcommand{\selfjoincallnoreturn}[1]{\funccallnoreturn{SelfJoin}{#1}}

\newcommand{\slopecallnoreturn}[1]{\funccallnoreturn{Slope}{#1}}
\newcommand{\costcallnoreturn}[1]{\funccallnoreturn{Cost}{#1}}

\newcommand{\nncallnoreturn}[1]{\funccallnoreturn{NN}{#1}}
\newcommand{\fplengths}{\mathcal{X}}
\newcommand{\fplength}[1]{x_{#1}}
\newcommand{\fplengthi}{\fplength{i}}

\newcommand{\fplengthmax}{\overset{\uparrow}{x}}
\newcommand{\fplengthavg}[1]{\overset{-}{x}\overset{{\left(#1\right)}}{}}
\newcommand{\fplengthavgj}{\fplengthavg{j}}
\newcommand{\mplengths}{\mathcal{Y}}
\newcommand{\mplength}[1]{y_{#1}}
\newcommand{\mplengthi}{\mplength{i}}

\newcommand{\hist}{\mathcal{H}}
\newcommand{\histbin}[1]{h_{#1}}
\newcommand{\histbine}{\histbin{e}}
\newcommand{\histbineprime}{\histbin{e'}}

\newcommand{\costset}{\mathcal{V}}
\newcommand{\costsetvalue}{v}
\newcommand{\costsetavg}{\overset{-}{v}}
\newcommand{\transformationcost}{t}
\newcommand{\degreeofisolation}[1]{\overset{\downarrow}{g}\overset{{\left(#1\right)}}{}}
\newcommand{\degreeofisolationj}{\degreeofisolation{j}}

%% file: 000abstract.tex
How could we have an outlier detector that works even with \textit{nondimensional} data, and ranks \textit{together} both singleton microclusters~(`one-off' outliers) and nonsingleton microclusters by their anomaly scores? How to obtain scores that are \textit{principled} in one \textit{scalable} and \textit{`hands-off'} manner? Microclusters of outliers indicate coalition or repetition in fraud activities, etc.; their identification is thus highly desirable. This paper presents \method: a new algorithm that detects microclusters by leveraging our proposed `Oracle' plot (\fplengthname versus \mplengthname). We study $31$ real and synthetic datasets with up to $1$M data elements to show that \method is the only method that answers both of the questions above; and, it outperforms $11$~other methods, especially when the data has nonsingleton microclusters or is nondimensional. We also showcase \method's ability to detect meaningful microclusters in graphs, fingerprints, logs of network connections, text data, and satellite imagery.
For example, it found a $30$-elements microcluster of confirmed `Denial of Service' attacks
in the network logs,
taking only \textit{$\mathit{\sim\hspace{-0.5mm}3}$ minutes} for $222$K data elements on a stock desktop.

%% file: 010introduction.tex
\begin{figure*}[ht]
    \begin{center}
        \begin{tabular}{c}
           \includegraphics[width=0.65\textwidth]{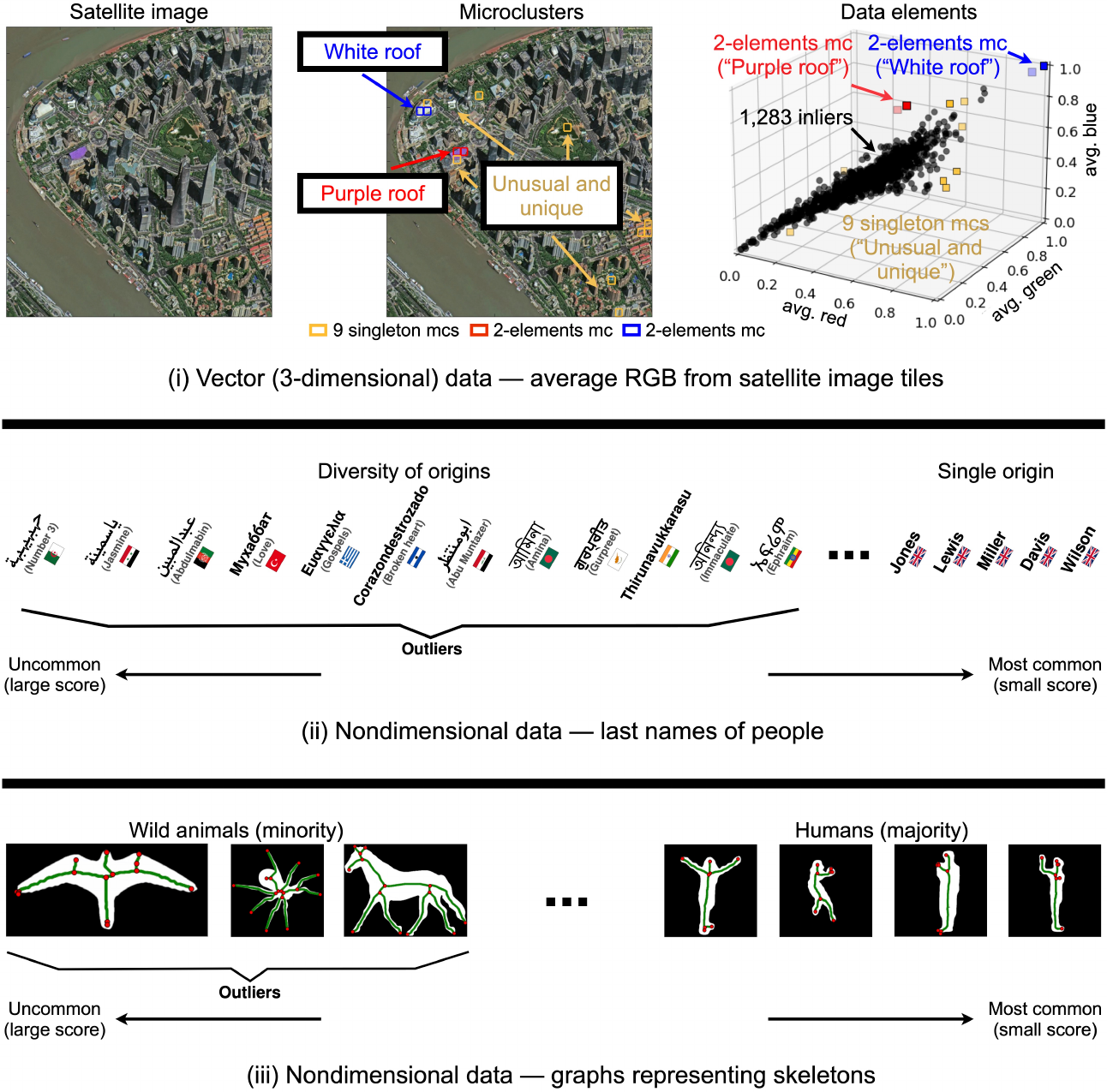} \\
        \end{tabular}
        \vspace{-2.5mm}
        \caption{
        \textbf{\method is unsupervised, and it ALSO works on nondimensional data:} (i)~on vector, $3$d data from a satellite image of Shanghai -- it spots two $2$-elements microclusters of unusually colored roofs, and a few other outliers; on nondimensional data of last names (ii) and skeletons (iii) -- it gives high anomaly scores to the few nonenglish names and skeletons of wild animals. (best viewed in color)
        \label{fig:crownJewel}
        \vspace{-12mm}
        }
    \end{center}
\end{figure*}

How could we have a method that detects microclusters of outliers even in a nondimensional dataset? How to rank together both singleton~(`one-off' outliers) and nonsingleton microclusters according to their anomaly scores? Can we define the scores in a principled way? Also, how to do that in a scalable and `hands-off' manner? Outlier detection has many applications and extensive literature~\cite{DBLP:books/sp/Aggarwal2013, Campos2016, Orairetal2010}.
The discovery of microclusters of outliers is among its most challenging tasks.
It happens because these outliers have close neighbors that make most algorithms fail~\cite{gen2out, DMCA, SCiForest}.
For example, see the red elements in the plots of Fig.~\ref{fig:crownJewel}(i) and Fig.~\ref{fig:axioms}.
Microclusters are critical for settings such as fraud detection and prevention of 
coordinated terrorist attacks, to name a few, because they indicate coalition or repetition as compared to `one-off' outliers. For example, a microcluster (or simply `mc', for short) can be formed from: 
\begin{enumerate*}[label=(\roman*)]
    \item frauds exploiting the same vulnerability in cybersecurity;
    \item reviews made by bots to illegitimately defame a product in e-commerce, or;
    \item unusual purchases of a hazardous chemical product made by ill-intended people.
\end{enumerate*}
Its discovery and comprehension is, therefore, highly desirable.

We present a new microcluster detector named~\textbf{\method} -- from \underline{\textbf{M}}icro\underline{\textbf{c}}luster \underline{\textbf{Catch}}.
The main idea is to leverage our proposed `Oracle' plot, which is a plot of \fplengthname versus \mplengthname with the latter being the distance from a cluster of data elements to its nearest neighbor. 
Our goals~are:
\begin{compactenum}[{G}1.]
    \item \textbf{\generalinput:} to work with any metric dataset, including nondimensional ones, such as sets of
    graphs, short texts, fingerprints, sequences of DNA, documents, etc.
    \item \textbf{\generaloutput:} to rank singleton (`one-off' outliers) and nonsingleton mcs \textit{together}, by their anomalousness.
    \item \textbf{\principled:} to obey axioms.
    \item \textbf{\scalable:} to be subquadratic on the number of elements.
    \item \textbf{\handsoff:} to be automatic with no manual tuning. 
\end{compactenum}

We studied $31$ datasets with up to $1$M elements to show that \method achieves all five goals, 
while $11$ of the closest state-of-the-art competitors fail.
Fig.~\ref{fig:crownJewel} showcases \method's ability to process dimensional and \textit{nondimensional} data:
\begin{enumerate*}[label=(\roman*)]
\item On vector, $3$d data from a satellite image of Shanghai, it spots two $2$-elements mcs of buildings with unusually colored roofs, and; a few other outliers.
On nondimensional data of last names~\item and skeletons~(iii), it gives high anomaly scores to the few nonenglish names and wild-animal skeletons.
\end{enumerate*}
The details of this experiment are given later; see Sec.~\ref{sec:exp}.

%% file: 021problem.tex
\begin{Problem}[\textbf{Main problem}]\label{def:problem}
It is as follows:
\begin{compactitem}
    \item Given: 
    \begin{compactitem}
        \item a metric dataset $\dataset~=~\{\point{1},\dots \point{\numpoints}\}$, where $\pointi$ is a data element (`point', in a dimensional case);
        \item a distance/dis-similarity function $\distfunc{\pointi}{\pointiprime}$
        (e.g., Euclidean~/~$L_p$ for dimensional data; provided by domain expert for non-dimensional data).
    \end{compactitem}
    \item Find:
        \begin{enumerate*}[label=\textbf{\underline{(\roman*)}}]
        \item a set of disjoint microclusters $\mcs=\{\mc{1}, \dots$ $ \mc{\nummcs} \}$, ranked most-strange-first, and 
        \item the set of corresponding anomaly scores $\scores = \{\score{1}, \dots \score{\nummcs} \}$
        \end{enumerate*}
        \item to match human intuition (see
        the axioms in Fig.~\ref{fig:axioms}).
\end{compactitem}
\end{Problem}

\vspace{0.2mm}
For ease of explanation, from now on, we shall describe our algorithm using the term `point' for each data element.
However, notice that the algorithm only needs a distance function between two elements -- \underline{NOT} coordinates.

%% file: 022relatedwork.tex
There is a huge literature on outlier and microcluster detection.
However, as we show in Tab.~\ref{tab:salesman}, only our \method meets
all the specifications. Next, we go into more detail.

\paragraph*{\hspace{-5mm} \underline{Related work vs. goals}} Outlier detection has many applications, including finance~\cite{finance, finance2}, manufacturing~\cite{manufactoring, manufactoring2}, environmental monitoring~\cite{environmental, environmental2}, to name a few. It is thus covered by extensive literature~\cite{DBLP:books/sp/Aggarwal2013,Campos2016,Orairetal2010}.
The existing methods can be categorized in various ways, for example, based on the measures they employ, including density-, depth-, angle- and distance-based methods, or by their modeling strategy, such as statistical-modeling, clustering, ensemble, etc., among others.

\input{TAB/specs}

\vspace{-0.5mm}
This section presents the related work in a nontraditional way.
We describe it considering the goals of our introductory section as we understand an ideal method should work with a \generalinput~(G1) and give a \generaloutput~(G2) that is \principled~(G3) in a \scalable~(G4) and \handsoff~(G5)~way.

\vspace{-0.5mm}
\paragraph*{\hspace{-5mm} \underline{\generalinput -- goal G1}}
Many methods fail w.r.t.~G1.
It includes famous isolation-based detectors, e.g., iForest~\cite{iForest}, Gen2Out~\cite{gen2out}, and~SCiForest~\cite{SCiForest}, other tree-based methods, such as
XTreK~\cite{XTreK} and
DIAD~\cite{DIAD}, hash-based approaches, e.g., Sparkx~\cite{Sparkx},
angle-based ones, like ABOD/FastABOD~\cite{ABOD},
some clustering-based methods, e.g., KMeans{-}{-}~\cite{Chawla:2013} and PLDOF~\cite{Pamula:2011},
and even acclaimed deep-learning-based detectors, such as RDA~\cite{RDA}, DOIForest~\cite{DOIForest}, and Deep SVDD~\cite{DeepSVDD}.
They all require access to explicit feature values.
Note that \textbf{embedding} may allow these methods to work on nondimensional data.
Turning elements of a metric or non-metric space into low-dimensionality vectors that preserve the given distances is exactly the problem of multi-dimensional scaling~\cite{BorgAndGroenen2005} and more recently t-SNE and UMAP~\cite{UMAP}. However, these strategies have two disadvantages:
\begin{enumerate*}[label=(\roman*)]
\item they are \textit{quadratic} on the count of elements~\cite{tSNE}, and; 
\item they require as input the embedding dimensionality.
\end{enumerate*}
Distinctly, density- and distance-based detectors -- as well as some clustering methods that detect outliers as a byproduct of the process, like DBSCAN~\cite{Ester:1996} and OPTICS~\cite{Ankerst:1999} -- may handle nondimensional data if adapted to work with a suitable distance function, and, ideally, also with a metric tree, like a Slim-tree~\cite{DBLP:journals/tkde/TrainaTFS02} or an M-tree~\cite{Mtree}.
Examples are LOCI~\cite{LOCI}, LOF~\cite{LOF}, GLOSH~\cite{GLOSH}, kNN-Out~\cite{kNN-Out}, DB-Out~\cite{DB-Out}, ODIN~\cite{ODIN}, LDOF~\cite{Zhang:2009}, and D.MCA~\cite{DMCA}.
However,~faster hypercube-based versions of some of these methods, e.g., ALOCI~\cite{LOCI}, require the features.

\paragraph*{\hspace{-5mm} \underline{\generaloutput -- goal G2}}
Most methods fail in~G2. 
They miss every mc whose points have close neighbors,
like ABOD, iForest, LOCI, Deep SVDD, RDA, GLOSH, kNN-Out, LOF, DB-Out, ODIN, DIAD, Sparkx, XTreK, and DOIForest;
or, fail to group these points into an entity with a score,
e.g., D.MCA, SCiForest, LDOF, PLDOF, DBSCAN, OPTICS, and KMeans{-}{-}.
Gen2Out is the only exception.

\vspace{-0.5mm}
\paragraph*{\hspace{-5mm} \underline{\principled -- goal G3}}
Goal~G3 regards the generation of scores in a principled manner for both singleton and nonsingleton mcs.
Gen2Out is the only method that provides scores for microclusters; thus, all other methods fail to achieve~G3.
Unfortunately, Gen2Out also fails w.r.t.~G3 as it does not identify nor obey any axiom for generating microcluster scores.
It does not obey axioms that we propose either.

\vspace{-0.5mm}
\paragraph*{\hspace{-5mm} \underline{\scalable -- goal G4}}
Some methods are scalable, like ALOCI, iForest, Gen2Out, SCiForest, PLDOF, KMeans{-}{-}, Sparkx, XTreK, DOIForest, RDA, and Deep SVDD.
They achieve~G4.
Distinctly, methods like DIAD, D.MCA, ABOD, FastABOD, GLOSH, LOCI, kNN-Out, LOF, DB-Out, ODIN, LDOF, DBSCAN, and OPTICS fail in~G4 as they are quadratic (or worse) on the count of~points. 

\vspace{-0.5mm}
\paragraph*{\hspace{-5mm} \underline{\handsoff -- goal G5}}
Goal~G5 regards the ability of a method to process an unlabeled dataset without manual tuning.
Methods that achieve~G5 are either hyperparameter free, like ABOD;
or, they have a default hyperparameter configuration to be used in all datasets, 
as it happens with FastABOD, Gen2Out, D.MCA, iForest, SCiForest, GLOSH, LOCI, and XTreK.
On the other hand, many methods fail w.r.t.~G5.
Examples are ALOCI, DB-Out, kNN-Out, LOF, LDOF, PLDOF, DBSCAN, OPTICS, KMeans{-}{-}, RDA, Deep SVDD, ODIN, DIAD, Sparkx, and DOIForest as they all require user-defined hyperparameter values.

\paragraph*{\hspace{-5mm} \underline{Conclusion: Only \method meets all the specifications}}
As mentioned earlier, and as shown in Tab.~\ref{tab:salesman},
\textit{\underline{only}} \method
fulfills all the specs.
Additionally, in contrast to several competitors, our method is deterministic and ranks outliers by their anomalousness. 
\method also returns explainable results
thanks to the plateaus of our `Oracle' plot (See  Sec.~\ref{sec:intuition}), which roughly correspond to the distance to the nearest neighbor.
Distinctly, black-box methods suffer on explainability.

%% file: TAB/specs.tex
\begin{table*}
\begin{center}
\caption{\textbf{\method matches all specs}, while the competitors miss one or more of the features. \label{tab:salesman}
\vspace{-2.5mm}
}
\resizebox{1.0\textwidth}{!}{
\begin{tabular}{l?c|c|c|c|c|c|c|c|c|c|c|c|c|c|c|c|c|c?c|c|c?c|c|c??c|}
       & \multicolumn{18}{c?}{\LARGE{Classic}} & \multicolumn{3}{c?}{\LARGE{Deep}} & \multicolumn{3}{c??}{\LARGE{Clustering}} & \\
       \diagbox{\LARGE Property}{\LARGE Method}
       & \rotatebox{70}{ABOD~\cite{ABOD}}
       & \rotatebox{70}{ALOCI~\cite{LOCI}} 
       & \rotatebox{70}{DB-Out~\cite{DB-Out}}
       & \rotatebox{70}{DIAD~\cite{DIAD}}
       & \rotatebox{70}{\underline{D.MCA~\cite{DMCA}}}
       & \rotatebox{70}{FastABOD~\cite{ABOD}} 
       & \rotatebox{70}{\underline{Gen2Out~\cite{gen2out}}} 
       & \rotatebox{70}{GLOSH~\cite{GLOSH}} 
       & \rotatebox{70}{iForest~\cite{iForest}} 
       & \rotatebox{70}{kNN-Out~\cite{kNN-Out}}
       & \rotatebox{70}{LDOF~\cite{Zhang:2009}}
       & \rotatebox{70}{LOCI~\cite{LOCI}} 
       & \rotatebox{70}{LOF~\cite{LOF}} 
       & \rotatebox{70}{ODIN~\cite{ODIN}}
       & \rotatebox{70}{PLDOF~\cite{Pamula:2011}}
       & \rotatebox{70}{SCiForest~\cite{SCiForest}}
       & \rotatebox{70}{Sparkx~\cite{Sparkx}}
       & \rotatebox{70}{XTreK~\cite{XTreK}}
       & \rotatebox{70}{Deep SVDD~\cite{DeepSVDD}}
       & \rotatebox{70}{DOIForest~\cite{DOIForest}}
       & \rotatebox{70}{RDA~\cite{RDA}}
       & \rotatebox{70}{DBSCAN~\cite{Ester:1996}} 
       & \rotatebox{70}{KMeans{-}{-}~\cite{Chawla:2013}}
       & \rotatebox{70}{OPTICS~\cite{Ankerst:1999}}
       & \rotatebox{70}{\LARGE{\underline{\textbf{\method}}}} \\
\hline
	{\Large G1. \generalinput} & & & \questionmarkcompetitor & & \questionmarkcompetitor & & & \questionmarkcompetitor & & \questionmarkcompetitor & \questionmarkcompetitor & \questionmarkcompetitor & \questionmarkcompetitor & \questionmarkcompetitor & & & & & & & & \questionmarkcompetitor & & \questionmarkcompetitor & \checkmarkmethod \\ 
    {\Large G2. \generaloutput} & & & & & & & \checkmarkcompetitor & & & & & & & & & & & & & & & & & & \checkmarkmethod \\ 
    {\Large G3. \principled~(group axioms)} & & & & & & & & & & & & & & & & & & & & & & & & & \checkmarkmethod \\ 
    {\Large G4. \scalable} & & \checkmarkcompetitor & & & & & \checkmarkcompetitor & & \checkmarkcompetitor & & & & & & \checkmarkcompetitor & \checkmarkcompetitor & \checkmarkcompetitor & \checkmarkcompetitor & \checkmarkcompetitor & \checkmarkcompetitor & \checkmarkcompetitor & & \checkmarkcompetitor & & \checkmarkmethod \\ 
    {\Large G5. \handsoff} & \checkmarkcompetitor & & & & \checkmarkcompetitor & \checkmarkcompetitor & \checkmarkcompetitor & \checkmarkcompetitor & \checkmarkcompetitor & & & \checkmarkcompetitor & & & & \checkmarkcompetitor & & \checkmarkcompetitor & & & & & & & \checkmarkmethod \\     
    {\Large Deterministic} & \checkmarkcompetitor & & \checkmarkcompetitor & & & \checkmarkcompetitor & & \checkmarkcompetitor & & \checkmarkcompetitor & \checkmarkcompetitor & \checkmarkcompetitor & \checkmarkcompetitor & \checkmarkcompetitor & & & & & & & & \checkmarkcompetitor & & \checkmarkcompetitor & \checkmarkmethod \\
    {\Large Explainable Results} & \checkmarkcompetitor & \checkmarkcompetitor & \checkmarkcompetitor & \checkmarkcompetitor & \checkmarkcompetitor & \checkmarkcompetitor & \checkmarkcompetitor & \checkmarkcompetitor & \checkmarkcompetitor & \checkmarkcompetitor & \checkmarkcompetitor & \checkmarkcompetitor & \checkmarkcompetitor & \checkmarkcompetitor & \checkmarkcompetitor & \checkmarkcompetitor & \checkmarkcompetitor & \checkmarkcompetitor & & & & \checkmarkcompetitor & \checkmarkcompetitor & \checkmarkcompetitor & \checkmarkmethod \\
    {\Large Rank Results} & \checkmarkcompetitor & \checkmarkcompetitor & \checkmarkcompetitor & \checkmarkcompetitor & \checkmarkcompetitor & \checkmarkcompetitor & \checkmarkcompetitor & \checkmarkcompetitor & \checkmarkcompetitor & \checkmarkcompetitor & \checkmarkcompetitor & \checkmarkcompetitor & \checkmarkcompetitor & \checkmarkcompetitor & \checkmarkcompetitor & \checkmarkcompetitor & \checkmarkcompetitor & \checkmarkcompetitor & \checkmarkcompetitor & \checkmarkcompetitor & \checkmarkcompetitor & & & & \checkmarkmethod \\ \hline
\end{tabular}
}
\vspace{-6mm}
\end{center}
\end{table*}

%% file: 030axioms.tex
How could we verify if a method reports scores in a principled way?
To this end, we propose reasonable axioms~that match human intuition and, thus, should be obeyed by any method when ranking microclusters w.r.t. their anomalousness.
Importantly, our axioms apply to singleton and also~to nonsingleton 
microclusters, so that both `one-off' outliers~and `clustered' outliers are included seamlessly into a \textit{\underline{single}} ranking.

The axioms state that the score $\scorej$ of a microcluster $\mcj$ depends on:
\begin{enumerate*}[label=(\roman*)]
\item the smallest distance between any point~$\pointi\in\mcj$ and this point's nearest inlier -- let this distance be known as the \degreeofisolationnamebold $\degreeofisolationj$ of $\mcj$, 
and;
\item the cardinality $\mid\hspace{-1mm}\mcj\hspace{-1mm}\mid$ of $\mcj$.
\end{enumerate*}
Hence, given any two microclusters that differ in one of these properties with all else being equal, we must~have:
\begin{compactitem}
\item \textbf{\iaxiom:} if they differ in the \degreesofisolationname, 
the furthest away microcluster has the largest score.
\item \textbf{\caxiom:} if they differ in the cardinalities, the less populous microcluster has the largest anomaly score.
\end{compactitem}

Fig.~\ref{fig:axioms} depicts our axioms in scenarios with inliers forming Gaussian-, cross- or arc-shaped clusters;
the green mc~(bottom) is always weirder, i.e., larger score, than the one in red (left).

\begin{figure*}[ht]
    \begin{center}
        \begin{tabular}{c}
           \includegraphics[width=0.65\textwidth]{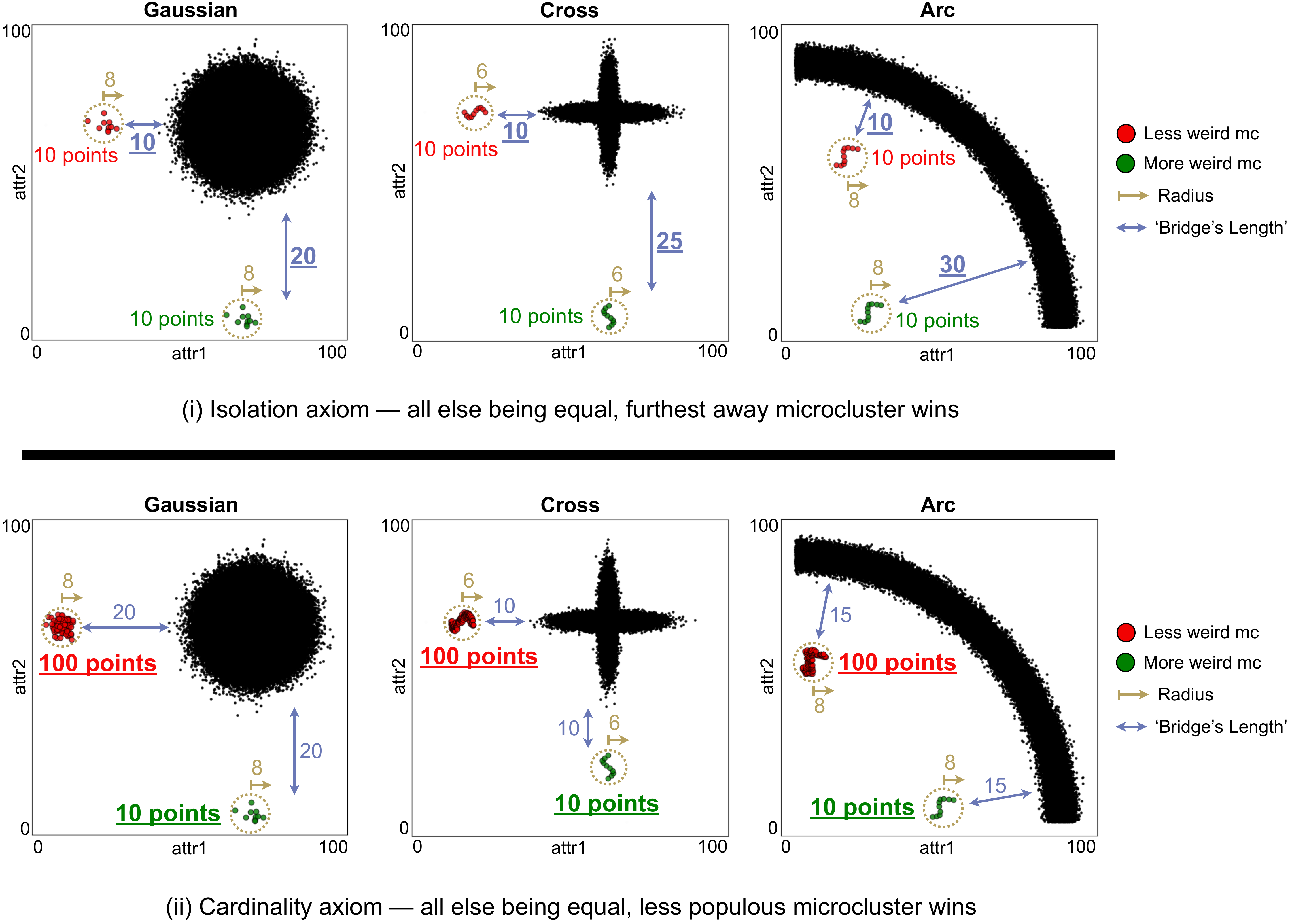} \\
        \end{tabular}
        \vspace{-2.5mm}
        \caption{\textbf{Proposed Axioms:} the green microcluster is always more weird, i.e., larger anomaly score. All else being equal, (i)~\iaxiom~--~furthest away microcluster wins; (ii)~\caxiom~--~less populous microcluster wins. (best viewed in color)
        \label{fig:axioms}
        \vspace{-12mm}
        }
    \end{center}
\end{figure*}

%% file: 040method.tex
\vspace{-2mm}
How could we have an outlier detector that works with a \generalinput~(G1) and gives a \generaloutput~(G2) that is \principled~(G3) in a \scalable~(G4) and \handsoff~(G5)~way?
Here we answer the question with \method.
To our knowledge, it is the first method that achieves all of these five goals.
We begin with the main intuition, and then detail our proposal.

\vspace{-3.5mm}
\subsection{Intuition \& the `Oracle' Plot}\label{sec:intuition}
\vspace{-2mm}

The high-level idea is to spot (i) points that are far from everything else (`one-off' outliers), and (ii) groups of points that form a microcluster, that is, they are close to each other, but far from the rest. The tough parts are how to quantify these intuitive concepts.
We propose to use the new `Oracle'~plot.

\paragraph*{\hspace{-5mm} \underline{`Oracle' plot}}
It focuses on plateaus formed in the count of neighbors of each point as the neighborhood radius varies.
This idea is shown in Fig.~\ref{fig:oracle:plot}.
We present a toy dataset in~\ref{fig:oracle:plot}(i), and its `Oracle' plot in~\ref{fig:oracle:plot}(ii).
For easy of understanding, let us consider five points of interest:
inlier~`A'~in black;
`halo-point' `B'~in orange;
`mc-point' `C' in green;
`halo-mc-point' `D' in violet, and;
`isolate-point' `E' in red.
Our `Oracle' plot groups inliers like `A' at its bottom-left part.
The other parts of the plot distinguish the outliers by type; see `B', `C', `D', and~`E'.
Note that the outliers `C' and `D' that belong to the mc of green/violet points are isolated at the top part of the plot.

The details are in Fig.\ref{fig:oracle:plot}(iii), where we plot the count of neighbors versus the radius for the points of interest.
Each point is counted together with its neighbors, so the minimum count is~$1$.
The large, blue curves give the average count for the dataset.
A plateau exists if the count of neighbors of a point remains (quasi) unaltered for two or more radii;
let this count be the \textit{height} of the plateau.
The \textit{length} of the plateau is the difference between its largest radius, and the smallest one.
Note that the counts for each point of interest form at least two plateaus:
the first plateau, and the last one, referring to small, and large radii respectively.
Middle plateaus may also~exist.

\begin{figure*}[ht]
    \begin{center}
        \begin{tabular}{c}
           \includegraphics[width=0.65\textwidth]{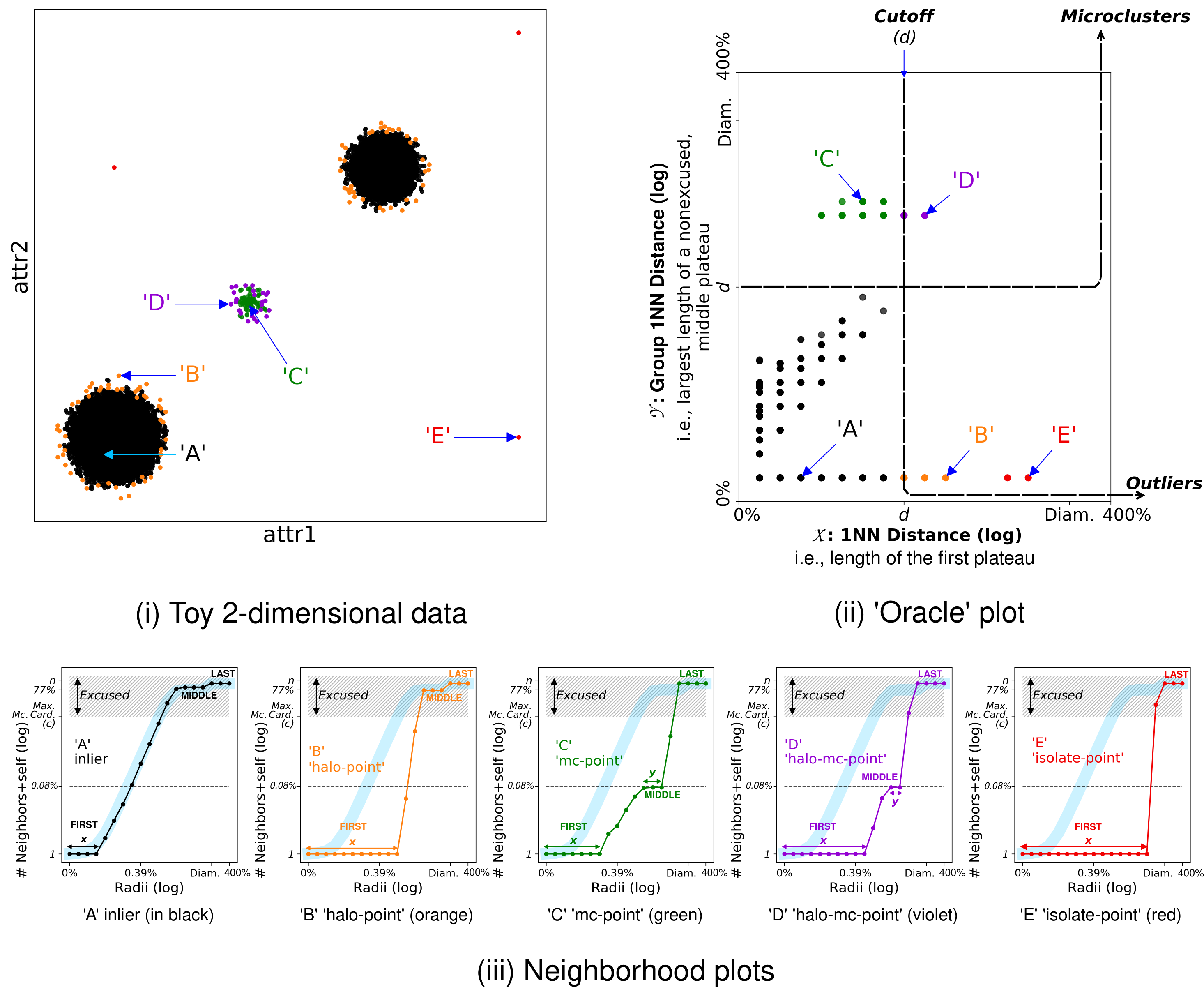} \\
        \end{tabular}
        \vspace{-2.5mm}
        \caption{\textbf{\underline{Intuition \& the `Oracle' plot:}} \method spots outliers in a dataset~(i) using our `Oracle' plot~(ii).
        The plot groups inliers like point `A' (in black) at its bottom-left, and distinguishes outliers by type;
        see `B'~(orange), `C'~(green), `D'~(violet), and~`E'~(red).
        Outliers `C' and `D' from the microcluster in green/violet are isolated at the top.
        It is made possible by capitalizing on plateaus formed in the count of neighbors of each point as the neighborhood radius varies; see examples in (iii).
        (best viewed in color)
        \label{fig:oracle:plot}
        \vspace{-6mm}
        }
    \end{center}
\end{figure*}

\paragraph*{\hspace{-5mm} \underline{`Plateaus' correspond to clusters}}
In fact, the plateaus follow~~a~~hierarchical~~clustering~~structure.~~Each~~plateau~~in~~the

\noindent count of neighbors of a point describes this point's cluster in one level of the hierarchy.
This is why the first, and the last plateaus always exist:
the first one regards a low level, where the point is a cluster of itself (or nearly so);
the last plateau refers to a higher level where the point, and (nearly) all other points cluster together.
Middle plateaus may be multiple, as the point may belong to clusters in many intermediary levels of the hierarchy.
Note that a plateau shows:
\begin{enumerate*}[label=(\roman*)]
\item the cluster's cardinality, and;
\item the cluster's distance to other points.
\end{enumerate*}
The plateau's height is the cardinality; its length is the distance.

\paragraph*{\hspace{-5mm} \underline{Examples of `plateaus'}}
The first plateaus in Fig.~\ref{fig:oracle:plot}(iii)
reveal that `A'~(in black) and `C' (green) are close to their nearest neighbors, while `E'~(red) is isolated -- see the length of each first plateau considering the log scale, and note the small lengths for `A' and `C', and the large one for `E'.
The middle plateaus show that: 
\begin{enumerate*}[label=(\roman*)]
\item `A' belongs to a populous, isolate cluster whose cardinality is $77\%$ of the dataset cardinality $\numpoints$ -- see its middle plateau whose length is large, as well as the height that is also large, and;
\item `C' is part of an isolate mc whose cardinality is $0.08\%$ of~$\numpoints$~--~note that its middle plateau has a large length, but a small height.
\end{enumerate*}
`E' does not belong to any nonsingleton cluster due to the absence of a middle plateau.

\paragraph*{\hspace{-5mm} \underline{Gory details: `excused' plateaus and the `Oracle' plot}}
\hfill Provided that we look for microclusters, we propose to \textit{excuse} plateaus of large height; i.e., to ignore clusters of large cardinality, say, larger than a
\maxsizename~$\maxsizedefault$.
See the `Excused' regions in the plots of Fig.~\ref{fig:oracle:plot}(iii).
Also, if any point happens to have two or more middle plateaus that are not excused, we only consider the one with the largest length, because the larger is the length of a plateau, the most isolated is the cluster that it describes.
From now on, if we refer to a plateau, we mean a \textit{nonexcused plateau}; and, if we refer to the middle plateau of a point, we mean this point's nonexcused, middle plateau of \textit{largest length}.

Our `Oracle' plot is then built from the plateaus' lengths.
We plot for each point~$\pointi$ the length~$\fplengthi$ of its first plateau versus the length~$\mplengthi$ of its middle plateau, using $\mplengthi=0$ when $\pointi$ 
has no middle plateau.
Importantly, $\fplengthi$ is approximately\footnote{~The exact distance could be found iff using an infinite set of radii,~which is unfeasible; plateaus would also have to have strictly unaltered neighbor~counts.\label{note:dnn}} the distance between~$\pointi$ and its nearest neighbor. 
Let us refer to $\fplengthi$ as the \fplengthnamebold of~$\pointi$.
On the other hand, $\mplengthi$ is approximately\footnote{~The exact distance could only be found for a point $\pointi$ at the center of the potential mc, and still being subject to the previous footnote's requirements.
} the largest distance between any potential, nonsingleton mc that contains $\pointi$, and the nearest neighbor of this cluster.
Thus, we refer to~$\mplengthi$ as the \mplengthnamebold of~$\pointi$.

Hence, the `X' axis of our `Oracle' plot represents the~possibility of each point to form a cluster of itself, that is, to be a singleton microcluster.
If a point $\pointi$ is far from any other point, then $\pointi$ has a larger $\fplengthi$ than it would have if it were close to another point.
Distinctly, the `Y' axis regards the possibility of each point to be in a nonsingleton microcluster.
If $\pointi$ has~a few close neighbors, i.e., fewer than $\maxsize$ neighbors, but it is far from other points, then $\pointi$ has a larger $\mplengthi$ than it would have if its close neighbors were many, or none.
Provided that $\fplengthi$ and $\mplengthi$ have both the same meaning~--~in the sense that each one is the distance between a potential microcluster $\mcj\suchthat\pointi\in\mcj$ and the cluster's nearest neighbor~--~they can be compared to a threshold~$\mindist$ to verify if $\pointi$ is an outlier; i.e., 
if either $\fplengthi$ or $\mplengthi$ is larger than or equal to~$\mindist$.
For instance, note in~Fig.~\ref{fig:oracle:plot}(ii) that a threshold $\mindist$ distinguishes outliers and inliers.
From now on, we refer to threshold~$\mindist$ as the \mindistnamebold.
\method obtains $\mindist$ automatically, in a data-driven way, as we show later.

\input{ALG/main}

\vspace{-1.25mm}
\subsection{\method in a Nutshell}
\vspace{-0.75mm}
\method is shown in Alg.~\ref{algo:main}.
Following the problem state-

\noindent ment from Probl.~\ref{def:problem},
it receives a dataset $\dataset = \left\{\point{1}, \dots \point{\numpoints} \right\}$ as input, and returns a set of microclusters $\mcs = \left\{\mc{1}, \dots \mc{\nummcs} \right\}$ together with their anomaly scores $\scores = \left\{\score{1}, \dots \score{\nummcs}\right\}$.
For applications that require a full ranking of the points (as well as for backward compatibility with previous methods),
\method also returns a set of scores per point $\scoresp = \left\{\scorep{1}, \dots \scorep{\numpoints} \right\}$, where $\scorepi \in \reals_{>0}$ is the score of point $\pointi$.

Our method has hyperparameters, for which we provide reasonable default values:
\numberradiiname~$\numberradiidefinition$;
\maxslopename~$\maxslopedefinition$, and;
\maxsizename $\maxsizedefinition$.
The default values $\numberradiidefault$, $\maxslopedefault$, and $\maxsizedefault$ were used in \textit{every} experiment reported in our paper\footnote{~Except for the experiments reported in Sec.~\ref{sec:sensitivity}, where we explicitly test the sensitivity to distinct hyperparametrization.\label{note:sensitivity}}.
It confirms \method's ability to be fully automatic.

In a nutshell, \method has four steps:
\begin{enumerate*}[label=(\Roman*)]
\item define neighborhood radii~--~see Lines~$1$-$3$ in Alg.~\ref{algo:main};
\item build `Oracle' plot~--~Line~$4$;
\item spot microclusters~--~Line~$5$, and;
\item compute anomaly scores~--~Line~$6$.
\end{enumerate*}
Step~I is straightforward.
We build a tree~$\tree$ for~$\dataset$, like an R-tree, M-tree, or Slim-tree\footnote{~M-trees and Slim-trees are for non-vector data; R-trees for disk-based vector data, and kd-trees for main-memory-based vector data. \label{note:tree}}.
Then, we estimate the diameter~$\diamdataset$ of $\dataset$ as the maximum distance between any child node (direct successor) of the root node of~$\tree$.
Finally, we define the set of radii~$\radii=\left\{\radius{1}, \dots \radius{\numberradii}\right\} = \left\{\frac{\diamdataset}{2^{\numberradii-1}},~\frac{\diamdataset}{2^{\numberradii-2}},~\dots~\frac{\diamdataset}{2^0}\right\}$ to be used when counting neighbors.
The next subsections detail Steps~II,~III, and~IV.

\vspace{-2.5mm}
\subsection{Build the `Oracle' Plot}
\vspace{-1.5mm}
Alg.~\ref{algo:build_o_plot} builds the `Oracle' plot.
It starts by counting neigh-

\noindent bors for each point w.r.t. a few radii of neighborhood -- see Lines~$1$-$3$.
Specifically, for each radius $\radiuse\in\radii$, we run a spatial self join algorithm $\selfjoinccallnoreturn{~}$ to obtain a set $\left\{\nn{e}{1}, \dots \nn{e}{\numpoints}\right\}$, where $\nnei$ is the count of neighbors (+ self) of~$\pointi$ considering radius $\radiuse$.
Any off-the-shelf spatial join algorithm can be used here.
However:
\begin{enumerate*}[label=(\roman*)]
\item the algorithm must be adapted to return only counts of neighbors, not pairs of neighboring points, and;
\item we consider the use of any join algorithm that can leverage a tree~$\tree$ to speed up the process.
\end{enumerate*}

Later, we use $\maxslope$, $\maxsize$, and the counts $\left\{\nn{1}{i},\dots \nn{\numberradii}{i}\right\}$ of each point~$\pointi$ to compute both the length $\fplengthi$ of its first plateau, i.e., the \fplengthname, and the length $\mplengthi$ of its middle plateau, i.e., the \mplengthname~--~see Lines~$4$-$7$.
The details are in Def.ns~\ref{def:plateau}-\ref{def:middleplateau}.
Note that we use $\fplengthi=0$ for every point $\pointi$ such that $\nn{1}{i}>1$, because, in these cases, the number of radii~$\numberradii$ is not large enough to uncover the first plateau of that particular point.
Similarly, we use $\mplengthi = 0$ for every point $\pointi$ that does not have a middle plateau.
The plateau lengths are then employed in Lines~$8$-$10$ to mount the `Oracle' plot $\op=\left( \left\{ \fplength{1}, \dots \fplength{\numpoints} \right\}, \left\{ \mplength{1}, \dots \mplength{\numpoints} \right\}\right)$, which is returned in Line~$11$.

\vspace{1mm}
\begin{definition}[\textbf{Plateau}]\label{def:plateau}
A plateau $\plateaui = \left[\radiuse,\radiuseprime\right]$ of a point $\pointi$ is a maximal range of radii where the count of neighbors of $\pointi$ remains unaltered, or quasi-unaltered according to a \maxslopename $\maxslope$. 
Formally, $\left[\radiuse,\radiuseprime\right]$ is a plateau of $\pointi$ if and only if $\radiuse,\radiuseprime\in\radii$ with $\radiuse < \radiuseprime$, and the slope
\vspace{-1mm}
\begin{equation*}
\slopecallnoreturn{e'',~i}=\frac{\text{log}\left(\nn{e'' + 1}{i}\right) - \text{log}\left(\nn{e''}{i}\right)}{\text{log}\left(\radius{e'' + 1}\right) - \text{log}\left(\radius{e''}\right)}
\end{equation*}
\vspace{-1mm}
\noindent is smaller than or equal to~$\maxslope$ for every value $e''$, such that $\radius{e''}\in\radii$ and $e \leq e'' < e'$.
Also, it must be true that:
$\slopecallnoreturn{e-1,~i}>\maxslope$, if $e>1$, and;
$\slopecallnoreturn{e',~i}>\maxslope$, if $e'<\numberradii$.
The~\textbf{length} of a plateau $\plateaui = \left[\radiuse,\radiuseprime\right]$ is given by $\radiuseprime-\radiuse$; its \textbf{height} is $\nnei$,
which must be smaller than, or equal to a \maxsizename~$\maxsize$.
\end{definition}

\vspace{1mm}
\begin{definition}[\textbf{First Plateau}]\label{def:firstplateau}
\hfill Among \hfill every \hfill plateau \\
$\plateaui~=~\left[\radiuse,\radiuseprime\right]$ of a point $\pointi$, 
the first plateau is the only one 
that has \textbf{height one}, that is, $\nnei==1$.
Let us use $\fplateaui$ to refer specifically to the first plateau of~$\pointi$.
Similarly, we use $\fplengthi$ to denote the length of~$\fplateaui$, which is the \fplengthnamebold~of~$\pointi$.
\end{definition}

\vspace{1mm}
\begin{definition}[\textbf{Middle Plateau}]\label{def:middleplateau}
Among every plateau $\plateaui=\left[\radiuse,\radiuseprime\right]$ of~$\pointi$ such that $\radiuseprime\neq\diamdataset$ and $\nnei>1$,
the middle plateau is the one with the \textbf{largest length}~$\radiuseprime-\radiuse$.
We use $\mplateaui$ to refer 
to the middle plateau of~$\pointi$.
Similarly, we use $\mplengthi$ to denote the length of~$\mplateaui$, which is the \mplengthnamebold of~$\pointi$.
\end{definition}

\vspace{-1mm}
\subsection{Spot the Microclusters}

Once the `Oracle' plot is built, how to
\begin{enumerate*}[label=(\roman*)]
\item spot the outliers,
and~then
\item group the ones that are nearby each other?
\end{enumerate*}
Alg.~\ref{algo:spot:mcs} gives~the answers. It has two steps:
to compute the \mindistname~$\mindist$
so to distinguish outliers from inliers 
as shown in Fig.~\ref{fig:oracle:plot}(ii);
and~then, to gel the outliers into mcs, that is, to assign each outlying point to the correct cluster.
The details are as follows.

\input{ALG/build_oracle_plot}

\input{ALG/spot-mcs}

\subsubsection{\underline{Compute the \mindistname}}

How to let the data dictate the correct \mindistname? 
Ideally, we want a method which is hands-off without requiring any parameters. The first solution that comes to mind is $k$ standard deviations with $k$ equals $3$. 
Can we get rid of the $k$ parameter too?

\paragraph*{\hspace{-5mm} \underline{\mindistname comes from compression: insight}}
Our insight is to use Occam’s razor~\cite{grunwald2004tutorial} and formally the Minimum Description Length (MDL)~\cite{Rissanen1982}. 
MDL is a powerful way of regularizing and eliminating the need for parameters. The idea is to choose those parameter values that result in the best compression of the given dataset.
It is a well respected principle made popular by Jorma Rissanen~\cite{Rissanen1982}, Peter Gr\"unwald~\cite{grunwald2004tutorial} and others.

We compute the \mindistname~$\mindist$
by capitalizing on the set $\left\{\fplength{1}, \dots \fplength{\numpoints}\right\}$ of \fplengthsname;
that is, by leveraging the 
`X' axis of our `Oracle' plot\footnote{~Intuitively, it would be equivalent to get $\mindist$ by using the plot's `Y'~axis, i.e., $\left\{\mplength{1}, \dots \mplength{\numpoints}\right\}$. 
The `X' axis is chosen simply because we must pick an~option.}.
Importantly, it is expected that many points have small 
values in this axis,
while only a few points have larger values.
The small values
come mostly from inliers, such as the black point `A' in Fig~\ref{fig:oracle:plot}(i), but a few ones may come from outliers in the core of a non-singleton microcluster, like point `C' in green, because these points also have close neighbors.
Distinctly, larger values derive exclusively from outliers, e.g., `B'~(in orange), `D'~(violet)~and `E'~(red).
It allows us to compute~$\mindist$ in a data-driven way, by partitioning a histogram of \fplengthsname so to best separate the tall bins that refer to small distances from the short bins that regard large distances.
Intuitively, \mindistname~$\mindist$ is the minimum distance required between one microcluster and its nearest~inlier.

Fig.~\ref{fig:minimum:cost} shows how we compute~$\mindist$.
At the bottom, it presents the `X' axis projection of the `Oracle' plot 
shown in Fig.~\ref{fig:oracle:plot}(ii).
Note that the positioning of points `A'~(in black), `B'~(orange), `C'~(green), `D'~(violet) and `E'~(red) reflects the discussion of the previous paragraph.
The corresponding histogram of points is shown in the central part of Fig.~\ref{fig:minimum:cost}.
We refer to it as the \histnamebold.
As expected, the majority of the points is counted in bins referring to small distances in the histogram -- see the tall bins on its left side.
The \mindistname~$\mindist$ is the distance that best separates 
the tall bins from the short ones, where the former refer to small distances and the latter regard large distances.
We obtain it automatically from the data, by partitioning the \histname so to minimize the cost of compressing the partitions -- see the top part of Fig.~\ref{fig:minimum:cost}.
Besides being parameter-free, our solution is grounded in the same concept of compression used later to generate scores, and, thus, it increases the coherence of our~method.

\begin{figure}[ht]
    \centering
    \includegraphics[width=0.7\linewidth]{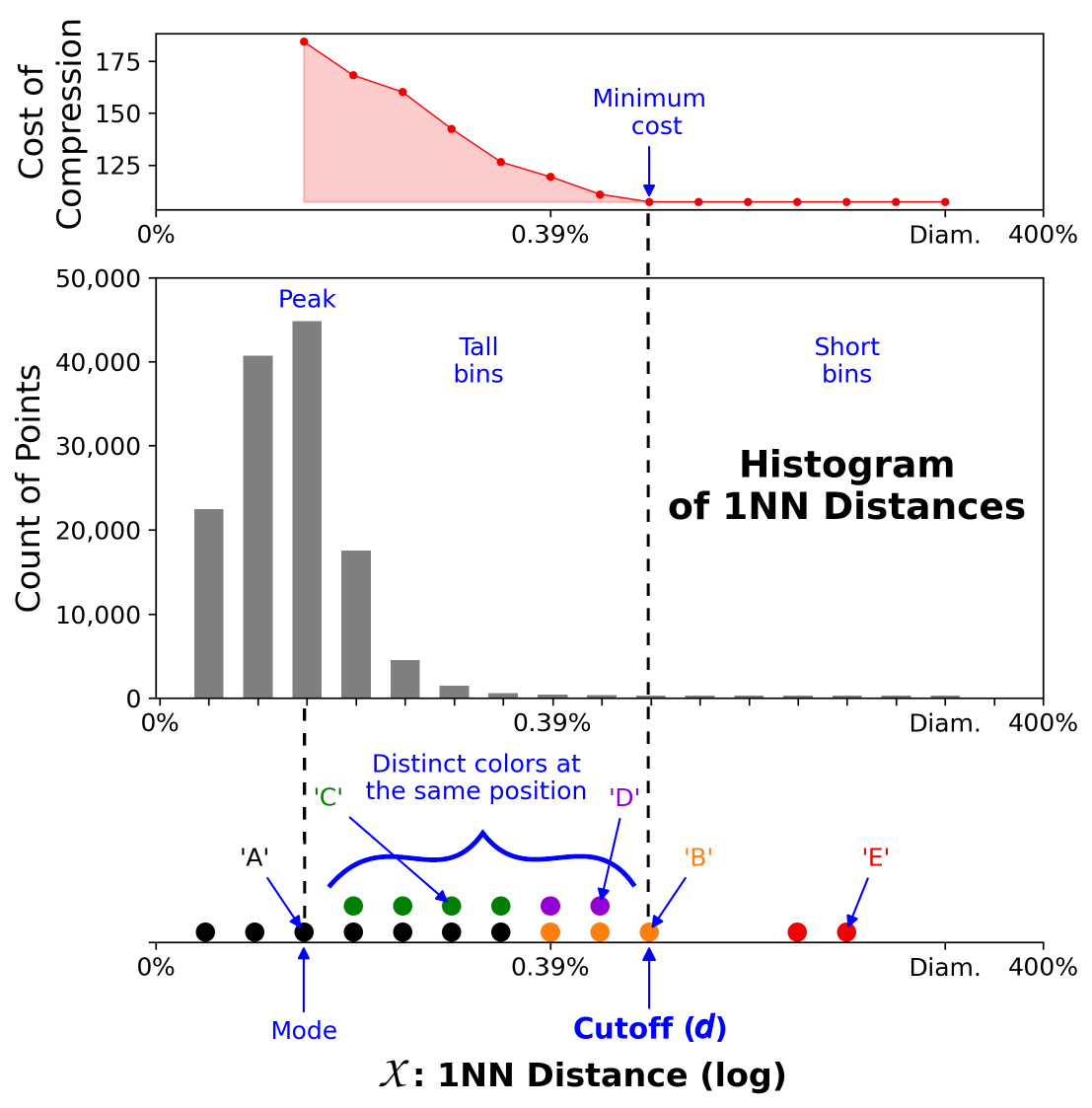}
    \vspace{-2.5mm}
    \caption{\textbf{\method obtains the \mindistnamebold $\mathbf{\mindist}$ automatically}, by partitioning a histogram of \fplengthsname so to best separate tall and short bins.
    It is done by minimizing the cost of compressing the partitions. (best viewed in color)
    \label{fig:minimum:cost}
    }
\end{figure}

\paragraph*{\hspace{-5mm} \underline{\mindistname \hfill comes \hfill from \hfill compression: \hfill details}} \hfill
Here \hfill we \hfill provide 

\noindent the details of computing~$\mindist$.
As shown in Lines~$1$-$5$ of Alg.~\ref{algo:spot:mcs}, we follow Def.~\ref{def:histogram} to build the \histname $\hist=\left\{\histbin{1},\dots \histbin{\numberradii} \right\}$.
Then, in Line~$6$, we use Def.ns~\ref{def:cost}-\ref{def:mindist} to obtain $\mindist$.
Specifically, we find the peak bin~$\histbineprime$.
It regards the distance/radius~$\radiuseprime$ most commonly seen between a point and its nearest neighbor.
Obviously, $\mindist$ must be larger than $\radiuseprime$.
Hence, we compute $\mindist$ considering only the bins in $\left\{\histbineprime,\dots \histbin{\numberradii} \right\}$.
They are analyzed to find the best cut position $e$
that maximizes the homogeneity of values in subsets $\left\{\histbin{e'},\ldots \histbin{e-1}\right\}$~and $\left\{\histbin{e},\ldots \histbin{\numberradii}\right\}$, so to best separate the tall from the short bins.

\vspace{1mm}
\begin{definition}[\histnamebold]\label{def:histogram} \ \ The \histname 
is defined as a set $\hist~=~\left\{\histbin{1},\dots \histbin{\numberradii} \right\}$, in which each bin $\histbine$ is computed as follows:
$\histbine~=~\mid\left\{ \pointi \in \dataset~\suchthat~\fplengthi == \radiuse \right\}\mid$.
\end{definition}

\vspace{1mm}
\begin{definition}[\textbf{\costcallname}]\label{def:cost}
The \MakeLowercase{\costcallname} of a nonempty set $\costset\subset\naturals_{\geq 0}$
with average $\costsetavg$ is
\begin{align*}
\costcallnoreturn{\costset}=&\bigl<\mid\costset\mid\bigr>+\biggl<1+\left \lceil{\costsetavg} \right \rceil \biggr>+ 
        \sum_{\costsetvalue\in\costset} \biggl<1+\left \lceil{\mid\costsetvalue - \costsetavg\mid}\right \rceil\biggr>
\end{align*}
\noindent , where 
$\bigl<~\bigr>$ is the universal code length for integers\footnote{~It can be shown that $\bigl<z\bigr> \approx \text{log}_2\left(z\right) + \text{log}_2\left(\text{log}_2\left(z\right)\right)+\dots$, where $z\in\naturals_{\geq 1}$ and only the positive terms of the equation are retained~\cite{DBLP:conf/kdd/ChakrabartiPMF04}. This is the optimal length, if we do not know the range of values for $z$ beforehand.\label{note:logstar}}.
\end{definition}

\vspace{1mm}
\begin{definition}[\mindistnamebold]\label{def:mindist}
The \mindistname $\mindist$ is defined as
\begin{align*}
\mindist~=~\radiuse \in \radii~\suchthat~e~\text{minimizes}~&\costcallnoreturn{\left\{\histbineprime,\ldots \histbin{e-1}\right\}} + \\ &\costcallnoreturn{\left\{\histbin{e},\ldots \histbin{\numberradii}\right\}}
\end{align*}
\noindent, where $\left\{\histbineprime,\ldots \histbin{e-1}\right\}$ and $\left\{\histbin{e},\ldots \histbin{\numberradii}\right\}$ are subsets of $\hist$, and $e'$ is chosen so that radius~$\radiuseprime\in\radii$ is the mode of $\left\{\fplength{1}, \dots \fplength{\numpoints}\right\}$.
\end{definition}
\vspace{1mm}

To obtain $e$, based on the principle of MDL, 
we check partitions $\left\{\histbineprime,\ldots \histbin{e''-1}\right\}$ and $\left\{\histbin{e''},\ldots \histbin{\numberradii}\right\}$ for all possible cut positions~$e''$.
The idea is to compress each possible partition, representing it by its cardinality, its average, and the differences of each of its values to the average. 
A partition with high homogeneity of values allows good compression, as its differences to the average are small, and small numbers need less bits to be represented than large ones do.
The best cut position $e$ is, therefore, the one that creates the partitions that compress best.
Note in Def.~\ref{def:cost} that we add ones to some values whose code lengths $\bigl<~\bigr>$ are required, so to account for zeros. 
\mindistname $\mindist$ is then obtained as $\mindist=\radiuse$, without depending on any input from the user.
It allows us to identify the set $\outliers=\left\{\pointi\in\dataset~\suchthat~\fplengthi\geq\mindist~\lor~\mplengthi\geq\mindist\right\}$ of all outliers.

\vspace{1mm}
\subsubsection{\underline{Gel the outliers into microclusters}}

Given the outliers, how to cluster them?
Alg.~\ref{algo:spot:mcs} provides the answer.
We use the `Y' axis of the `Oracle' plot to isolate outliers of nonsingleton mcs into a set~$\nodes=\{\pointi\in\outliers\suchthat\mplengthi\geq\mindist\}$; see Line~$8$. 
Then, in Lines~$9$-$15$, we group these points using the plot's `X' axis.
Every outlier in $\nodes$ must be grouped together with its nearest neighbor; thus, we identify the largest \fplengthname $\fplengthmax~=~\text{max}_{i}~\fplength{i}\suchthat\pointi\in\nodes$, and use it to specify a threshold that rules if each possible pair of points from $\nodes$ is close enough to be grouped together.
Provided that the \fplengthsname are approximations, the threshold itself is the smallest radius larger than $\fplengthmax$, that is, radius~$\radius{e+1}$ from Line~$12$.
It avoids having a point and its nearest neighbor in distinct clusters.

The nonsingleton mcs are then identified by spotting connected components in a graph~$\graph~=~\left(~\nodes,~\edges~\right)$, where $\nodes$ is the set of nodes, and $\edges\subseteq\nodes\cartprod\nodes$ is the set of edges. The edges are obtained from any off-the-shelf spatial self join algorithm $\selfjoincallnoreturn{~}$ that returns pairs of nearby points from $\nodes$ -- see Line~$12$.
Lastly, in Lines~$16$-$19$, we recognize each outlier in $\outliers\setminus\nodes$ as a cluster of itself, and return the final set of mcs~$\mcs$.

\subsection{Compute the Anomaly Scores}

Given the mcs, how to get scores that obey our axioms?

\paragraph*{\hspace{-5mm} \underline{Scores come from compression: insight}}
We quantify the anomalousness of each mc according to 
how much it can be compressed when described in terms of the nearest inlier.
Fig.~\ref{fig:score} depicts~this idea.
To describe a microcluster $\mcj$
we would store its cardinality $\mid\mcj\mid$ and 
the identifier $i\in\left\{1,2,\dots \numpoints\right\}$ of the nearest inlier $\pointi$.
See Items~$\raisebox{-2pt}{\text{\textcolor{red}{\Large \ding{172}}}}$ and $\raisebox{-2pt}{\text{\textcolor{red}{\Large \ding{173}}}}$ in Fig.~\ref{fig:score}.
Then, we would use $\pointi$ as a reference to
describe the point $\pointiprime\in\mcj$ that is the closest to it.
To this end, we would store the differences (e.g., in each feature if we have vector data) between~$\pointi$~and~$\pointiprime$; see~$\raisebox{-2pt}{\text{\textcolor{red}{\Large \ding{174}}}}$.
Point~$\pointiprime$ would in turn be the reference to describe its nearest neighbor $\pointiprimeprime\in\mcj$, which would later serve as a reference to describe one other close neighbor from $\mcj$, thus following a recursive process that would lead us to describe every point of $\mcj$; see~$\raisebox{-2pt}{\text{\textcolor{red}{\Large \ding{175}}}}$.
Importantly,
in this representation,
the cost per point -- that is, the total number of bits used to describe $\mcj$ divided by~$\mid\hspace{-1mm}\mcj\hspace{-1mm}\mid$ -- reflects the axioms of Fig.~\ref{fig:axioms}.
A large \degreeofisolationname increases the cost per point due to~$\raisebox{-2pt}{\text{\textcolor{red}{\Large \ding{174}}}}$.
Also, the larger the cardinality, the smaller the cost per point. It is because the costs of~$\raisebox{-2pt}{\text{\textcolor{red}{\Large \ding{172}}}}$, 
$\raisebox{-2pt}{\text{\textcolor{red}{\Large \ding{173}}}}$ and~$\raisebox{-2pt}{\text{\textcolor{red}{\Large \ding{174}}}}$
are diluted with points.
Hence, the cost per point appropriately quantifies the anomalousness of~$\mcj$.

\paragraph*{\hspace{-5mm} \underline{Scores come from compression: details}}
Alg.~\ref{algo:scoring:mcs} computes the scores.
We begin by finding 
the \degreeofisolationname~$\degreeofisolationj$
of each~$\mcj$.
To this end, we compute the distance $\dnii$ between each outlier $\pointi\in\outliers$ and its nearest inlier; see Lines~$1\text{-}12$.
Specifically, we run a join between $\outliers$ and $\dataset\setminus\outliers$ for each $\radiuse\in\radii$.
Any spatial join algorithm can be used here, but it must be adapted to return counts of neighbors, not pairs of points.
Each $\dnii$ is then the largest radius for which $\pointi$ has zero neighbors. 
The \degreeofisolationname~$\degreeofisolationj$ is finally found for each $\mcj$ in Line~$17$; 
it is the smallest distance $\dnii\suchthat\pointi\in\mcj$.

\begin{figure}[t]
    \centering
    \includegraphics[width=0.5\linewidth]{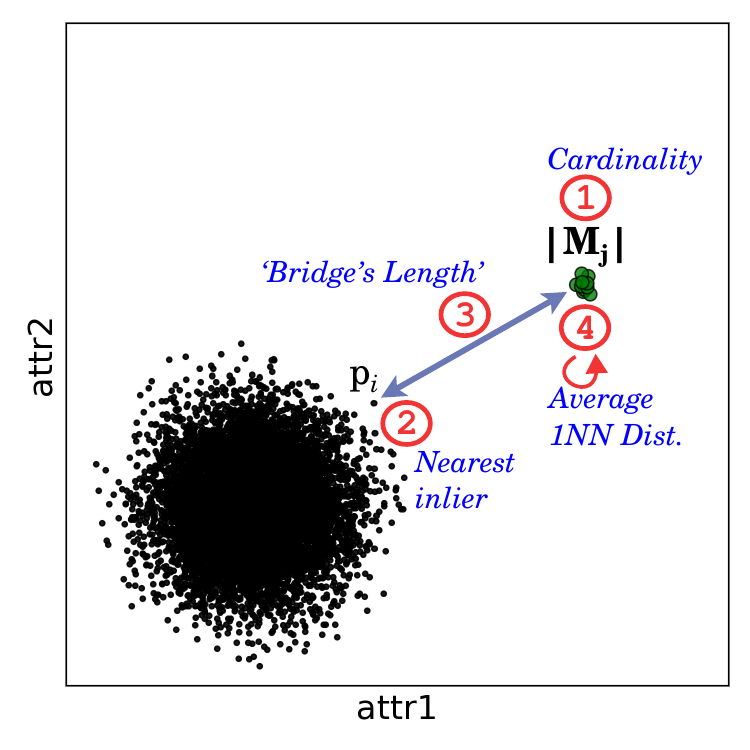}
    \vspace{-2.5mm}
    \caption{\textbf{\method's scores} quantify how much each mc $\mcj$ is compressed when it is described in terms of the nearest inlier $\pointi$. (best viewed in~color)
    \label{fig:score}
    }
\end{figure}

\vspace{1mm}
\begin{definition}[\textbf{\scorename}]\label{def:score}
The score $\scorej$ of $\mcj$ is defined as
\begin{align*}
\scorej~=~\frac{\text{\textcolor{red}{\Large \ding{172}}}~~\raisebox{2pt}{$+$}~~\text{\textcolor{red}{\Large \ding{173}}}~~\raisebox{2pt}{$+$}~~\text{\textcolor{red}{\Large \ding{174}}}~~\raisebox{2pt}{$+~~\left(\mid\mcj\mid-1\right)\mult$}~\text{\textcolor{red}{\Large \ding{175}}}}{\mid\mcj\mid}
\end{align*}
\noindent where
\begin{align*}
\raisebox{-2pt}{\text{\textcolor{red}{\Large \ding{172}}}}~&=~\bigl<~\mid\mcj\mid~\bigr>\text{\SideCommentPos{42mm}{Cardinality}}\\
\raisebox{-2pt}{\text{\textcolor{red}{\Large \ding{173}}}}~&=~\bigl<~\numpoints~\bigr>\text{\SideCommentPos{47mm}{Nearest inlier}}\\
\raisebox{-2pt}{\text{\textcolor{red}{\Large \ding{174}}}}~&=~\transformationcost~\mult~\Biggl<~\frac{
\degreeofisolationj
}{\radius{1}}~\Biggr>\text{\SideCommentPos{29mm}{`Bridge's Length'}}\\
\raisebox{-2pt}{\text{\textcolor{red}{\Large \ding{175}}}}~&=~\transformationcost~\mult~\Biggl<~1~+~\left \lceil{\frac{
\fplengthavgj
}{\radius{1}}}\right\rceil~\Biggr>\text{\SideCommentPos{6.5mm}{Average \fplengthname}}
\end{align*}
\noindent with $\fplengthavgj=\text{avg}_{i}~\fplengthi\suchthat\pointi\in\mcj$. The~\transformationcostnamebold~$\transformationcost$~is
\begin{equation*}
    \transformationcost=
    \begin{cases}
        \text{Data dimensionality} &\text{\hspace{-4mm}, if \textbf{vector space}} 
        \vspace{0.1in}
        \\
        \bigl<3\bigr>+\bigl<\text{\# distinct chars.}\bigr> &\text{\hspace{-3.75mm}, if \textbf{words using}}
        \\
        +~\bigl<\text{\# chars. longest word}\bigr> &\text{\hspace{-1.75mm}\textbf{edit distance}}
        \vspace{0.1in}
        \\
        \text{Cost to transform a point into another} &\text{\hspace{-2.5mm}, if \textbf{other space}}
        \\
        \text{point that is one unit of distance away} &
    \end{cases}
\end{equation*}
\noindent and $\bigl<~\bigr>$ is the universal code length for integers\footref{note:logstar}.
\end{definition}
\vspace{1mm}

After that, we follow suit with the idea of Fig.~\ref{fig:score} and compute the scores using Def.~\ref{def:score}.
Item~$\raisebox{-2pt}{\text{\textcolor{red}{\Large \ding{172}}}}$ in Def.~\ref{def:score} is the cost of storing the cardinality $\mid\hspace{-1.25mm}\mcj\hspace{-1.25mm}\mid$.
Item~$\raisebox{-2pt}{\text{\textcolor{red}{\Large \ding{173}}}}$ is the cost of storing
the identifier $i\in\left\{1,2,\dots\numpoints\right\}$ of the nearest inlier~$\pointi$. 
We always consider the worst-case scenario here (that is, when $i=\numpoints$) to standardize the cost and also to avoid any dependence on the order of points in~$\dataset$.
Item~$\raisebox{-2pt}{\text{\textcolor{red}{\Large \ding{174}}}}$ reflects the cost of storing the differences between $\pointi$ and the point $\pointiprime\in\mcj$ that is the closest to it.
To work with vector and nondimensional data, we approximate this cost using the distance between $\pointi$ and $\pointiprime$, which is~$\degreeofisolationj$; it is then~normalized using $\radius{1}$. We also use the \transformationcostname~$\transformationcost$ of the metric space; i.e., the cost to transform a point into another point that is one unit of distance away.
For example, in a vector space, $\transformationcost$~is the dimensionality because we must know the differences between $\pointi$ and $\pointiprime$ in each one of the features.
As for another example, when analyzing words using the edit distance, $t$ is the cost of describing:
\begin{enumerate*}[label=(\roman*)]
\item the transformation to be performed, among the three options insertion, deletion and replacement;
\item the new character to be inserted or to replace the existing one, and;
\item the position where the transformation must occur.
\end{enumerate*}
Finally, Item $\raisebox{-2pt}{\text{\textcolor{red}{\Large \ding{175}}}}$ refers to the cost of storing the differences between each one of the remaining $\mid\hspace{-1mm}\mcj\hspace{-1mm}\mid - 1$ points to be described and its point of reference.
The cost is once more approximated using distances to work with any metric data.

For applications requiring a full ranking of the points (and for backward compatibility with other methods), we also have a score $\scorepi$ for each $\pointi$. 
To this end, we follow previous ideas and propose a score representing the cost of describing $\pointi$ in terms of the nearest inlier; 
see Lines~$13$-$15$ and~$21$-$24$~in~Alg.~\ref{algo:scoring:mcs}.

\vspace{-2mm}
\subsection{Time and Space Complexity}
\label{sec:complexity}

\begin{lemma}[\textbf{Time Complexity}]\label{lem:timecomplexity}
The time complexity of \method is $O\left(\numpoints~\mult~\numpoints^{1-\frac{1}{\intdim}}\right)$, where 
$\intdim$ is the intrinsic (correlation fractal) dimension\footnote{~We only need distances to compute the fractal dimension~$\intdim$, which is how quickly the number of neighbors grows with the distance~\cite{DBLP:conf/pods/FaloutsosK94}.
It can be computed even for nondimensional~data, for example, using string-editing distance for strings of last names or tree-editing distance for skeleton-graphs.
Moreover, Traina Jr.\cite{DBLP:journals/tkde/TrainaTFS02} show how to quickly estimate the fractal dimension of a nondimensional dataset requiring subquadratic time.
} of~$\dataset$. 
\end{lemma}
\begin{proof}
\method is presented in Alg.~\ref{algo:main}.
It builds a tree~$\tree$ for~$\dataset$ in $O\left(\numpoints~\mult~\text{log}\left(\numpoints\right)\right)$ time, and then computes~$\diamdataset$ and $\radii$ in a negligible time. 
Later, it spots microclusters by calling functions
$\buildopdef$ from Alg.~\ref{algo:build_o_plot}, 
$\spotmcsdef$ from Alg.~\ref{algo:spot:mcs} and 
$\scoremcsdef$ from Alg.~\ref{algo:scoring:mcs}, sequentially.
Therefore, the overall cost of \method is the larger between $O\left(\numpoints~\mult~\text{log}\left(\numpoints\right)\right)$ and the costs of Algs.~\ref{algo:build_o_plot},~\ref{algo:spot:mcs} and~\ref{algo:scoring:mcs}.

Alg.~\ref{algo:build_o_plot} counts neighbors by running $\numberradii$ self-joins for~$\dataset$.
Because $\numberradii$ is small, the complexity of this step is the same as that of one single self-join.
The self-join finds neighbors for each of $\numpoints$ points.
A rough, upper-bound 
estimation for its runtime is thus
$O\left(\numpoints \mult \nncallnoreturn{\numpoints}\right)$,
where $\nncallnoreturn{\numpoints}$ is time taken by each nearest neighbor search.
According to Pagel, Korn and Faloutsos~\cite{DBLP:conf/icde/PagelKF00},
$\nncallnoreturn{\numpoints}=O\left(\numpoints^{1-\frac{1}{\intdim}}\right)$, in which $\intdim$ is the intrinsic (correlation fractal) dimension of a vector dataset $\dataset$. 
For nondimensional data, Traina Jr. et al.~\cite{DBLP:journals/tkde/TrainaTFS02} demonstrate that the query cost also depends on $\intdim$.
It leads us to estimate the cost of counting neighbors as $O\left(\numpoints~\mult~\numpoints^{1-\frac{1}{\intdim}}\right)$. 
Later, Alg.~\ref{algo:build_o_plot} finds plateaus in $O\left(\numpoints\right)$ time, and then it mounts 
$\op$ in a negligible time.
Hence, the cost of Alg.~\ref{algo:build_o_plot} is $O\left(\numpoints~\mult~\numpoints^{1-\frac{1}{\intdim}}\right)$.

Algs.~\ref{algo:spot:mcs} and~\ref{algo:scoring:mcs} 
scan $\dataset$ in $O\left(\numpoints\right)$ time. They also process the set of outliers $\outliers$, which takes a negligible time because $\mid\outliers\mid$ is small.
Thus, the time required by each algorithm is~$O\left(\numpoints\right)$.

As the total cost is the larger between $O\left(\numpoints~\mult~\text{log}\left(\numpoints\right)\right)$ and the costs of Algs.~\ref{algo:build_o_plot},~\ref{algo:spot:mcs} and~\ref{algo:scoring:mcs},
it comes to~$O\left(\numpoints~\mult~\numpoints^{1-\frac{1}{\intdim}}\right)$.
We consider $\intdim>1$ because it is the most expected scenario.
$\blacksquare$
\end{proof}

\vspace{1mm}
\paragraph*{\hspace{-5mm} \underline{Note}} In our experience, real data often have fractal dimension $\intdim$ smaller than~$20$~\cite{DBLP:conf/icde/PagelKF00, DBLP:conf/sbbd/TrainaTWF00}.
As shown in~\cite{DBLP:journals/tkde/TrainaTFS02}, several nondimensional data have small $\intdim$.
Thus, \method should be subquadratic on the count of points for most real applications.

\vspace{1mm}
\begin{lemma}[\textbf{Space \hfill Complexity}]\label{lem:spacecomplexity}
\hfill The \hfill space \hfill complexity \hfill of \\ \method is given by $O\left(n\right)$.
\end{lemma}
\begin{proof}
\method receives a set $\dataset$ as the input, and returns sets $\mcs$, $\scores$ and $\scoresp$ as the output.
The largest data structures employed to this end are the `Oracle' plot $\op$ and a tree~$\tree$ created for $\dataset$ in Alg.~\ref{algo:main}.
$\dataset$, $\scoresp$, $\op$ and $\tree$ require $O\left(\numpoints\right)$ space each.
$\mcs$ and $\scores$ have negligible space requirements, as they require $O\left(\mid\outliers\mid\right)$ space with a small $\mid\outliers\mid$.
Consequently, the space complexity of \method is $O\left(\numpoints\right)$.
$\blacksquare$
\end{proof}

\input{ALG/scoring-mcs}

\vspace{-3mm}
\subsection{Implementation} \label{sec:implementation}
\vspace{-2mm}

The main cost of \method regards counting neighbors to build the `Oracle' plot.
As shown in Lines~$1$-$3$ of Alg.~\ref{algo:build_o_plot}, we can count neighbors using a didactic strategy that is easy to follow.
Nevertheless, when using this strategy, we count the neighbors of every point w.r.t. all radii, which is unnecessary because we only need a count $\nnei$ if it is smaller than or equal to~$\maxsize$;
see the `Excused' regions in Fig.~\ref{fig:oracle:plot}(iii).
Our actual implementation follows a sparse-focused, speed-up principle:
\begin{compactitem}
\item \textbf{Sparse-focused principle}:
run a self-join for $\dataset$ with $\radius{1}$. Then, for each $\radiuse\in\left\{\radius{2},\dots\radius{\numberradii}\right\}$, from smallest to largest, run a join (not self) between~$\left\{\pointi\in\dataset~\suchthat~\nn{e-1}{i}\leq\maxsize\right\}$ and $\dataset$, thus computing only the counts required.
\end{compactitem}

Other speed-up principles we employ for joins are:
\begin{compactitem}
\item \textbf{Count-only principle}: do not materialize pairs of neighbors~(e.g., by leveraging `compact similarity joins'~\cite{DBLP:conf/icde/BryanEF08}) 
because we only need counts of neighbors, not the actual pairs of points. It applies to the joins in Algs.~\ref{algo:build_o_plot} and~\ref{algo:scoring:mcs}.
\item \textbf{Using-index principle}: use a tree, like an R-tree, M-tree, or Slim-tree\footref{note:tree}.
It applies to the joins in Algs.~\ref{algo:build_o_plot},~\ref{algo:spot:mcs} and~\ref{algo:scoring:mcs}.
\item \textbf{Small-radii-only principle}: don't run a join for radius~$\radius{\numberradii}$.
Since $\radius{\numberradii}=\diamdataset$, 
we already know all points are neighbors of each other.
It applies to the joins in Algs.~\ref{algo:build_o_plot},~\ref{algo:spot:mcs} and~\ref{algo:scoring:mcs}.
\end{compactitem}

%% file: ALG/main.tex
\begin{algorithm}[ht]
{
\caption{$\methoddef$}
\label{algo:main}
\raggedright
\Input Dataset $\dataset=\left\{\point{1},\dots \point{\numpoints}\right\}$; \\
        \hspace*{8.5mm} \numberradiiname $\numberradii$; \SideComment{\textbf{Default}: $\numberradiidefault$} \\
        \hspace*{8.5mm} \maxslopename $\maxslope$; \SideComment{\textbf{Default}: $\maxslopedefault$} \\
        \hspace*{8.5mm} \maxsizenameshort $\maxsize$; \SideComment{\textbf{Default}: $\maxsizedefault$} \\        
\Output Microclusters $\mcs=\left\{\mc{1}, \dots \mc{\nummcs}\right\}$; \\
        \hspace*{11mm} Scores per microcluster $\scores = \left\{\score{1}, \dots \score{\nummcs}\right\}$; \\
        \hspace*{11mm} Scores per point $\scoresp = \left\{\scorep{1}, \dots \scorep{\numpoints}\right\}$;
\begin{algorithmic}[1]
\LineComment{Define the neighborhood radii}
\State Build tree $\tree$ for $\dataset$;  \SideCommentPos{-0.5mm}{Like a Slim-tree, M-tree, or R-tree\footref{note:tree}}
\State Estimate diameter $\diamdataset$ of $\dataset$ from $\tree$;
\State $\radii = \left\{\radius{1}, \radius{2}, \dots \radius{\numberradii} \right\} = \left\{\frac{\diamdataset}{2^{\numberradii-1}},~\frac{\diamdataset}{2^{\numberradii-2}},~\dots~\frac{\diamdataset}{2^0}\right\}$;
\LineComment{Build the `Oracle' plot}
\State $\buildopcall$; \SideComment{Alg.~\ref{algo:build_o_plot}}
\LineComment{Spot the microclusters}
\LineSecComment{$\mcs = \{\mc{1},\dots \mc{\nummcs}\}$, where $\mcj \subset \dataset$}
\State $\spotmcscall$; \SideComment{Alg.~\ref{algo:spot:mcs}}
\LineComment{Compute the anomaly scores}
\LineSecComment{$\scores = \left\{\score{1}, \dots \score{\nummcs}\right\}$, where $\scorej$ refers to $\mcj$}
\LineSecComment{$\scoresp = \left\{\scorep{1}, \dots \scorep{\numpoints}\right\}$, where $\scorepi$ refers to $\pointi$}
\State $\scoremcscall$; \SideComment{Alg.~\ref{algo:scoring:mcs}}
\State \Return $\mcs, \scores, \scoresp$;
\end{algorithmic}
}
\end{algorithm}

%% file: ALG/build_oracle_plot.tex
\begin{algorithm}[t]
{
\caption{$\buildopdef$}
\label{algo:build_o_plot}
\raggedright
\Input Dataset $\dataset=\left\{\point{1},\dots \point{\numpoints}\right\}$; \\
        \hspace*{8.5mm} Tree $\tree$; \\
        \hspace*{8.5mm} Radii $\radii = \left\{\radius{1}, \dots \radius{\numberradii}\right\}$; \\
        \hspace*{8.5mm} \maxslopename $\maxslope$; \\
        \hspace*{8.5mm} \maxsizename $\maxsize$; \\        
\Output `Oracle' plot $\op=\left( \left\{ \fplength{1}, \dots \fplength{\numpoints} \right\} ,\left\{ \mplength{1}, \dots \mplength{\numpoints} \right\}\right)$;
\begin{algorithmic}[1]
\LineComment{Count the neighbors}
\LineSecComment{$\nnei= $ \# neighbors (+ self) of $\pointi$ regarding radius $\radiuse$}
\For{$e = 1, \dots \numberradii$} \SideComment{Run a join per radius}
    \State $\selfjoinccall{\left\{\nn{e}{1},\dots \nn{e}{\numpoints}\right\}}{\tree, \radiuse}$;
\EndFor
\LineComment{Find the plateaus}
\For{$\pointi \in \dataset$}
    \LineSecCommentPos{9mm}{Compute the \fplengthname}
    \State $\fplengthi = $ use $\maxslope$, $\maxsize$, and $\left\{\nn{1}{i},\dots \nn{\numberradii}{i}\right\}$ to compute 
    \Statex \hspace*{12.5mm} the length of the first plateau of $\pointi$; \SideComment{Def.~\ref{def:firstplateau}}
    \LineSecComment{Compute the \mplengthname}
    \State $\mplengthi = $ use $\maxslope$, $\maxsize$, and $\left\{\nn{1}{i},\dots \nn{\numberradii}{i}\right\}$ to compute
    \Statex \hspace*{12mm} the length of the middle plateau of $\pointi$; \SideComment{Def.~\ref{def:middleplateau}}
\EndFor
\LineComment{Mount the `Oracle' plot}
\State $\fplengths = \left\{\fplength{1},\dots \fplength{\numpoints}\right\}$; \SideComment{`X' axis}
\State $\mplengths = \left\{\mplength{1},\dots \mplength{\numpoints}\right\}$; \SideComment{`Y' axis}
\State $\op = \left(~\fplengths,~\mplengths~\right)$; \SideComment{`Oracle' plot}
\State \Return $\op$;
\end{algorithmic}
}
\end{algorithm}

%% file: ALG/spot-mcs.tex
\begin{algorithm}[ht]
{
\caption{$\spotmcsdef$}
\label{algo:spot:mcs}
\raggedright
\Input Dataset $\dataset=\left\{\point{1},\dots \point{\numpoints}\right\}$; \\
        \hspace*{8.5mm} `Oracle' plot $\op=\left( \left\{ \fplength{1}, \dots \fplength{\numpoints} \right\} ,\left\{ \mplength{1}, \dots \mplength{\numpoints} \right\}\right)$; \\
        \hspace*{8.5mm} Radii $\radii = \left\{\radius{1}, \dots \radius{\numberradii}\right\}$; \\
\Output Microclusters $\mcs= \left\{\mc{1},\dots \mc{\nummcs}\right\}$;
\begin{algorithmic}[1]
\LineComment{Compute the \mindistname $\mindist$}
\LineSecComment{Build the \histname $\hist$}
\State $\hist=\left\{\histbin{1},\dots \histbin{\numberradii} \right\} = \left\{0,\dots 0\right\}$;
\For{$\pointi \in \dataset$}
    \State $e=e'\in\left\{1, \dots \numberradii\right\}~\suchthat~\fplengthi == \radiuseprime$; \SideComment{Find bin}
    \State $\histbine = \histbine + 1$; \SideComment{Count point $\pointi$ in bin $\histbine$}
\EndFor
\LineSecComment{\textbf{Data-driven} computation of $\mindist$}
\State $\mindist = $ use $\hist$ and $\radii$ to compute the \mindistname; \SideComment{Def.~\ref{def:mindist}}
\LineComment{Gel the outliers into microclusters}
\State $\outliers=\left\{\pointi\in\dataset\suchthat\fplengthi\geq\mindist\lor\mplengthi\geq\mindist\right\}$; \SideComment{All outliers}
\LineSecComment{Gel nonsingleton microclusters}
\State $\nodes = \{\pointi \in \outliers\suchthat\mplengthi \geq \mindist\}$; \SideComment{Large \mplengthname}
\State Build tree $\tree$ for $\nodes$;
\State $\fplengthmax~=~$max$_{i}~\fplength{i}~\suchthat~\pointi \in \nodes$; \SideComment{Largest \fplengthname}
\State $e~=~e'\in\left\{1, \dots \numberradii\right\}~\suchthat~\radius{e'}==\fplengthmax$;
\State $\selfjoincall{\edges}{\tree, \radius{e+1}}$; \SideComment{Find neighbors}
\State $\graph~=~\left(~\nodes,~\edges~\right)$; \SideComment{Build neighborhood graph}
\State $\component{1},\dots \component{\nummcs'} =$ connected components of $\graph$, where 
\Statex \hspace*{15mm} $\componentj = \left(\nodescomponentj\subseteq\nodes,\edgescomponentj\subseteq\edges\right)$; \SideComment{Find components}
\State $\mcs = \{\mc{1},\dots \mc{\nummcs'}\}$; \SideComment{Nonsingleton mcs}
\LineSecComment{Gel singleton microclusters}
\For{$\pointi \in \outliers \setminus \nodes$}
        \State $\mcs = \mcs~\cup~\left\{\pointi\right\}$; \SideComment{Add singleton mc}
\EndFor
\State \Return $\mcs$;
\end{algorithmic}
}
\end{algorithm}

%% file: ALG/scoring-mcs.tex
\begin{algorithm}[t] 
{
\caption{$\scoremcsdef$}
\label{algo:scoring:mcs}
\raggedright
\Input Dataset $\dataset=\left\{\point{1},\dots \point{\numpoints}\right\}$; \\
        \hspace*{8.5mm} Microclusters $\mcs=\left\{\mc{1}, \dots \mc{\nummcs}\right\}$; \\ 
        \hspace*{8.5mm} `Oracle' plot $\op=\left( \left\{ \fplength{1}, \dots \fplength{\numpoints} \right\} ,\left\{ \mplength{1}, \dots \mplength{\numpoints} \right\}\right)$; \\
        \hspace*{8.5mm} Radii $\radii = \left\{\radius{1}, \dots \radius{\numberradii}\right\}$; \\
\Output Scores per microcluster $\scores = \left\{\score{1}, \dots \score{\nummcs}\right\}$; \\
        \hspace*{11mm} Scores per point $\scoresp = \left\{\scorep{1}, \dots \scorep{\numpoints}\right\}$;
\begin{algorithmic}[1]
\LineComment{Compute the distances to the nearest inliers}
\LineSecComment{For outliers}
\State $\outliers = \bigcup_{j=1}^{\nummcs} \mc{j}$; \SideComment{All outliers}
\State Build tree $\tree$ for $\outliers$; \SideComment{Tree for outliers}
\State Build tree $\treeprime$ for $\dataset \setminus \outliers$; \SideComment{Tree for inliers}
\For{$e = 1, \dots \numberradii$}
    \LineSecCommentPos{9mm}{$\nninei=$ \# of neighboring inliers of $\pointi$ w.r.t. $\radiuse$}
    \State $\joinccall{\left\{\nnin{e}{i} \suchthat \pointi~\text{in}~\tree \right\}}{\tree,\treeprime, \radiuse}$; \SideComment{Join}
    \For{$\pointi$ in $\tree$}
        \If{$\nninei > 0$}
            \State $\dnii = \radius{e-1}$; \SideComment{$\dnii= $ distance to nearest inlier}
            \State Remove $\pointi$ from $\tree$;
        \EndIf
    \EndFor
\EndFor
\LineSecComment{For inliers}
\For{$\pointi \in \dataset \setminus \outliers$}
    \State $\dnii = \fplengthi$; \SideComment{$\dnii$ is the \fplengthname of $\pointi$}
\EndFor
\LineComment{Compute the scores per microcluster}
\For{$\mcj \in \mcs$}
    \State $\degreeofisolationj=\text{min}_{i}~\dnii\suchthat\pointi\in\mcj$; \SideComment{\degreeofisolationname}
    \State $\scorej = $ use $\mcj$, $\radii$, $\op$, $\degreeofisolationj$ to get score of~$\mcj$; \SideComment{Def.~\ref{def:score}}
\EndFor
\State $\scores = \left\{\score{1},\dots,\score{\nummcs}\right\}$; \SideComment{Set of scores per microcluster}
\LineComment{Compute the scores per point}
\For{$\pointi \in \dataset$}
    \State $\scorepi=\Bigl<1 + \frac{\dnii}{\radius{1}}\Bigr>$; \SideComment{$\bigl<~\bigr>$ is the code length for integers\footref{note:logstar}}
\EndFor
\State $\scoresp = \left\{\scorep{1},\dots,\scorep{\numpoints}\right\}$; \SideComment{Set of scores per point}
\State \Return $\scores, \scoresp$;
\end{algorithmic}
}
\vspace{-1.6mm}
\end{algorithm}

%% file: 050experiment.tex
\vspace{-1.5mm}
We designed experiments to answer five questions:
\begin{compactenum}[{Q}1.]
\item {\bf \accurate}: How accurate is \method?
\item {\bf \principled}: Does \method obey axioms?
\item {\bf \scalable}: How scalable is \method?
\item {\bf \practical}: How well \method works on real data?
\item {\bf \handsoff}: Does \method need manual tuning?
\end{compactenum}

\paragraph*{\hspace{-5mm} \underline{Setup, code, competitors, and datasets}}
\method was coded in Java and C\raisebox{0.5pt}{$++$}.
The joins in Algs.~\ref{algo:build_o_plot}-\ref{algo:scoring:mcs} employ the approach of `compact similarity joins'~\cite{DBLP:conf/icde/BryanEF08}.

We compared \method with $11$ state-of-the-art competitors: 
ABOD, FastABOD, LOCI, ALOCI, DB-Out, LOF, iForest and ODIN, which are coded in Java under the framework ELKI~(\url{elki-project.github.io}), besides; 
Gen2Out, D.MCA and RDA whose original source codes in Python were used.
Tab.~\ref{tab:hyperparameters} has the hyperparameter values employed.
\method was always tested with its default configuration\footref{note:sensitivity}.
The competitors were carefully tuned following hyperparameter-setting heuristics widely adopted in prior works, such as in \cite{iForest, SCiForest, Bandaragodaetal2018, LOF, Tingetal2020, DMCA, gen2out, Campos2016, kNN-Out, RDA}.  
Non-deterministic competitors were run~$10$ times per dataset; we report the average results.

\begin{table}[b]
\centering
\caption{Hyperparameter configuration.
\vspace{-2.5mm}
}
\resizebox{0.7\columnwidth}{!}{
\begin{tabular}{l|l}
  \textbf{Method} & \textbf{Values used}\\
  \hline \hline
  ALOCI & $g \in \{10,~15,~20\}$,~~$nmin = 20$,~~$\alpha = 4$\\
  DB-Out & $r \in \{\diamdataset \mult 0.05,~\diamdataset \mult 0.1,~\diamdataset \mult 0.25,~\diamdataset \mult 0.5\}$\\
  D.MCA & $\psi \in \{2,~4,~8,~\dots~\text{min}(1,024,~\numpoints \mult 0.3)\}$,\\
  & $t\in\{2,~4,~8,~\dots~128\}$,~~$p=\numpoints\mult0.1$\\
  FastABOD & $k \in \{1,~5,~10\}$\\
  Gen2Out & $lb = 1$,~~$ub = 11$,~~$md \in \{2,3\}$,~~$t \in \{2,~4,~8,~\dots~128\}$\\
  iForest & $t \in \{2,~4,~8,~\dots~128\}$,~~$\psi \in \{2,~4,~8,~\dots~\text{min}(1,024,~\numpoints \mult 0.3)\}$\\
  LOCI & $r \in \{\diamdataset \mult 0.05,~\diamdataset \mult 0.1,~\diamdataset \mult 0.25,~\diamdataset \mult 0.5\}$,~~$nmin = 20$,~~$\alpha = 0.5$\\
  LOF & $k \in \{1,~5,~10\}$\\
  ODIN & $k \in \{1,~5,~10\}$\\
  RDA & $nlayers \in \{2,~3,~4\}$,~~$dimdecay \in \{1,~2,~4\}$,\\
  & $niter \in \{20,~50\}$,~~$\lambda \in \{10e-5,~7.5e-5,~10e-4\}$\\
  \hline \hline
  \method & $\numberradiidefault$,~~$\maxslopedefault$,~~$\maxsizedefault$\\
  \hline \hline
\end{tabular}
}
\label{tab:hyperparameters}
\end{table}

Tab.~\ref{tab:datasets} summarizes our data, which we describe as~follows:
\begin{compactitem}
\item \textbf{Last Names:} 
$5k$ names of people frequent in the US (inliers), and $50$ names frequent elsewhere~(outliers).
\item \textbf{Fingerprints:} ridges from $398$ full (inliers) and $10$ partial~(outliers) fingerprints.
\item \textbf{Skeletons:} skeleton graphs from $200$ human~(inliers) and $3$ wild-animal (outliers) silhouettes.
\item \textbf{Axioms:} synthetic data with Gaussian-, cross- and~arc-shaped inliers following each axiom as shown in~Fig.~\ref{fig:axioms}.
\item \textbf{Popular benchmark datasets:} benchmark data from many real domains.
Importantly, HTTP and Annthyroid are known to have nonsingleton microclusters~\cite{SCiForest}.
\item \textbf{Shanghai and Volcanoes:} average RGB values extracted from satellite image tiles. Outliers are unknown.
\item \textbf{Uniform and Diagonal:} $2$-, $4$-, $20$-, and $50$-dim. data that follow a uniform distribution, or form a diagonal line.
\end{compactitem}

In all cases we have that
\begin{enumerate*}[label=(\roman*)]
\item for vector data, we use the Eu- clidean distance (but any other $L_p$ metric would work),~and
\item for nondimensional datasets the distance function is given by a domain expert.
\end{enumerate*}
For example, string-editing or soundex encoding distance\cite{fuzzystrmatch} for strings, and mathematical morphology\cite{Vincent1989} or tree-editing distance\cite{PawlikAugsten2015} for shapes or skeleton~graphs.
\vspace{-1mm}

\begin{table}[t]
\centering
\caption{Summary of datasets.
\vspace{-2.5mm}
}
\resizebox{0.7\columnwidth}{!}{
\begin{tabular}{ll|l|l|l|l}
  & \textbf{Dataset} & \rotatebox{70}{\textbf{\# Points}} & \rotatebox{70}{\parbox{16mm}{\textbf{\# Features\\ (Emb. Dim.)}}} & \rotatebox{70}{\parbox{18mm}{\textbf{Intrinsic\\ (Fractal) Dim.}}} & \rotatebox{70}{\textbf{\% Outliers}}\\
  \hline \hline
  \vspace{1.2mm}
  \multirow{3}{*}{\rotatebox{90}{\textbf{Non-dim.\ }}} & Last Names & $5,050$ & -- & $5.3$ & $2.22$ \\
  \vspace{1.2mm}
  & Fingerprints & $408$ & -- & $8.0$ & $2.45$ \\
  \vspace{1.2mm}
  & Skeletons & $203$ & -- & $2.1$ & $1.47$ \\
  \hline \hline
  \vspace{0.5mm}
  \multirow{2}{*}{\rotatebox{90}{\textbf{Axs.\ }}} & Gauss., Cross, Arc (\iaxiomshort) & $\sim 1$ million & $2$ & $1.7$ & $0.002$ \\
  & Gauss., Cross, Arc (\caxiomshort) & $\sim 1$ million & $2$ & $1.7$ & $0.01$ \\
  \hline \hline
  \multirow{17}{*}{\rotatebox{90}{\textbf{Popular Benchmark Datasets}}} & HTTP & $222,027$ & $3$ & $1.2$ & $0.03$  \\
  & Shuttle & $49,097$ & $9$ & $1.8$ & $7.15$  \\
  & kddcup08 & $24,995$ & $25$ & $3.6$ & $0.68$  \\
  & Mammography & $7,848$ & $6$ & $1.4$ & $3.22$  \\
  & Annthyroid & $7,200$ & $6$ & $1.8$ & $7.41$ \\
  & Satellite & $6,435$ & $36$ & $3.0$ & $31.64$  \\
  & Satimage2 & $5,803$ & $36$ & $3.0$ & $1.22$  \\
  & Speech & $3,686$ & $400$ & $5.9$ & $1.65$  \\
  & Thyroid & $3,656$ & $6$ & $0.7$ & $2.54$  \\
  & Vowels & $1,452$ & $12$ & $0.8$ & $3.17$  \\
  & Pima & $526$ & $8$ & $2.2$ & $4.94$  \\
  & Ionosphere & $350$ & $33$ & $1.6$ & $35.71$  \\
  & Ecoli & $336$ & $7$ & $1.9$ & $2.68$  \\
  & Vertebral & $240$ & $6$ & $1.9$ & $12.5$  \\
  & Glass & $213$ & $9$ & $1.3$ & $4.23$  \\
  & Wine & $129$ & $13$ & $2.3$ & $7.75$  \\
  & Hepatitis & $70$ & $20$ & $1.5$ & $4.29$  \\
  & Parkinson & $50$ & $22$ & $1.4$ & $4$  \\
  \hline \hline
  \multirow{2}{*}{\rotatebox{90}{\textbf{Sat.}}} & Volcanoes & $3,721$ & $3$ & $1.8$ & Unknown \\
  & Shanghai & $1,296$ & $3$ & $1.4$ & Unknown \\
  \hline \hline
  \vspace{1mm}
  \multirow{2}{*}{\rotatebox{90}{\textbf{Synt.\ }}} & Uniform & $1$ million & $\left[2,50\right]$ & $\left[2.0,50.0\right]$ & $0$ \\
  \vspace{1mm}
  & Diagonal & $1$ million & $\left[2,50\right]$ & $1.0$ & $0$ \\
  \hline \hline
\end{tabular}
}
\label{tab:datasets}
\end{table}

\begin{figure}[t]
    \centering
    \begin{tabular}{c}
       \includegraphics[width=0.99\linewidth]{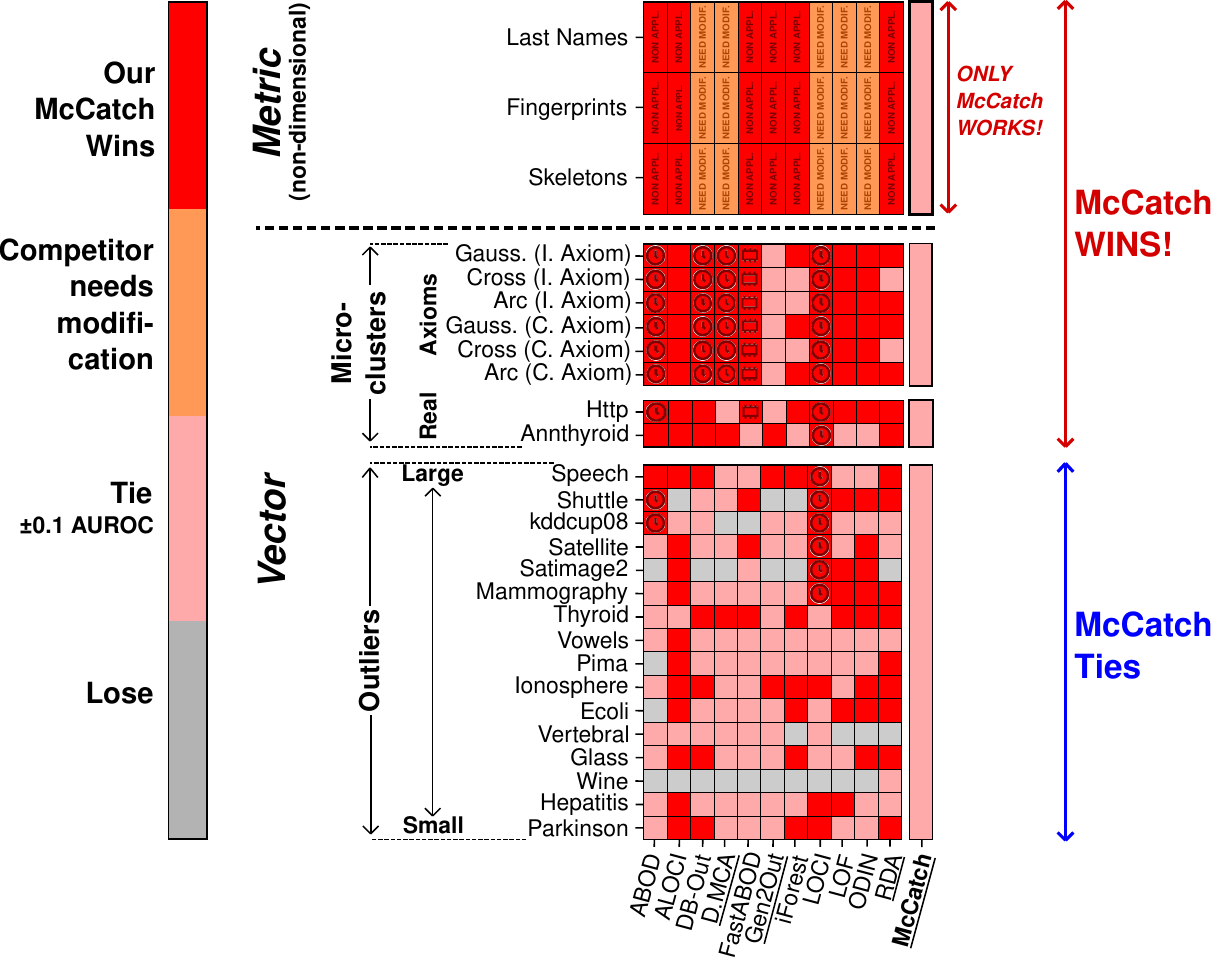} \\
    \end{tabular}
    \vspace{-2.5mm}
    \caption{\textbf{Q1. \method is \accurate \underline{\textit{and}} typically beats competition~(red):}
    Accuracy comparison.
    Top: \method wins in vector data with known non-singleton microclusters, and \textit{also} in nondimensional data.
    Bottom: our~method ties with the competitors in other cases.
    (best viewed in color)
    \label{fig:correctness}
    }
\end{figure}

\begin{figure*}[t]
    \begin{center}
        \begin{tabular}{c}
           \includegraphics[width=0.8\textwidth]{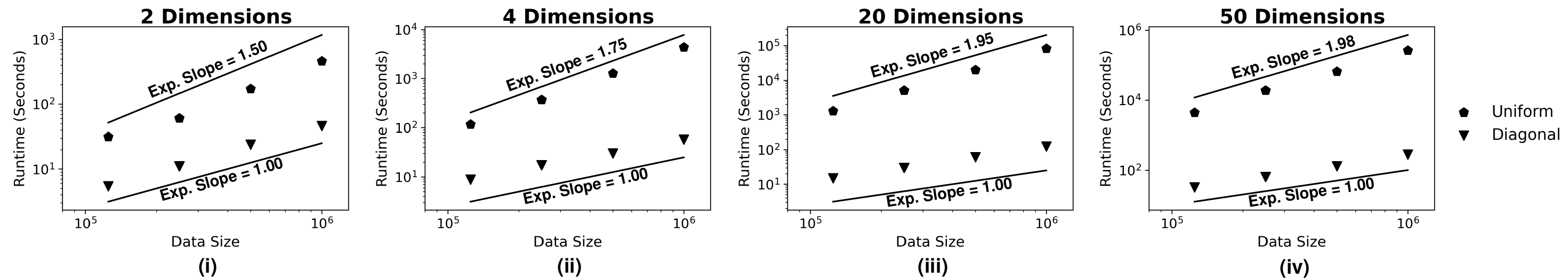} \\
        \end{tabular}
        \vspace{-2.5mm}
        \caption{\textbf{Q3. \method is \scalable:}
        Runtime vs. data size.
        Lines show expected slopes (Lemma~\ref{lem:timecomplexity}).
        Estimations and measurements agree very well:
        \method is subquadratic in all cases, despite embedding dimension.
        Note: Uniform in (iii) and (iv) have respectively \textit{\underline{fractal} dimension~$\mathit{20}$ and $\mathit{50}$!} 
        \label{fig:scalability} 
        }
        \vspace{-9mm}
    \end{center}
\end{figure*}

\vspace{-6mm}
\subsection{Q1. \method is \accurate}
\vspace{-2mm}

Fig.~\ref{fig:correctness} reports results regarding the accuracy of \method.
For every dataset where outliers are known, we compare the Area Under the Receiver Operator Characteristic curve~(AUROC) obtained by \method with that of each competitor.
All methods were evaluated according to the anomaly scores they reported per point.
Note that it was unfeasible to run D.MCA, ABOD, FastABOD, DB-Out and LOCI in some datasets;
they either required an excessive runtime (i.e., $>10$ hours) or an excessive RAM memory usage (i.e., $>30$ GB).
These two cases are respectively denoted by symbols~$\raisebox{-1pt}{\text{\includegraphics[height=1.5ex]{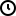}}}$
and~$\raisebox{-1pt}{\text{\includegraphics[height=1.5ex]{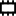}}}$.

\begin{table}
\begin{center}
\caption{\textbf{Accuracy Evaluation}\label{tab:correctness}
\vspace{-3mm}
}
\resizebox{1.0\columnwidth}{!}{
\begin{tabular}{l|c|c|c|c|c|c|c|c|c|c|c||c|}
       \diagbox{\textbf{Metric}}{\textbf{Method}}
       & \rotatebox{70}{ABOD}
       & \rotatebox{70}{ALOCI}
       & \rotatebox{70}{DB-Out}
       & \rotatebox{70}{\underline{D.MCA}}
       & \rotatebox{70}{FastABOD}
       & \rotatebox{70}{\underline{Gen2Out}}
       & \rotatebox{70}{iForest}
       & \rotatebox{70}{LOCI}
       & \rotatebox{70}{LOF}
       & \rotatebox{70}{ODIN}
       & \rotatebox{70}{\underline{RDA}}
       & \rotatebox{70}{\underline{\textbf{\method}}} \\
\hline
    H. Mean Rank (AUROC) & \cellcolor{red!60}$2.2$ & $4.6$ & $5.1$ & $4.4$ & $2.8$ & $2.8$ & $3.5$ & $6.0$ & $4.9$ & $4.9$ & $3.6$ & \cellcolor{red!100}$1.8$ \\
    H. Mean Rank (AP) & \cellcolor{red!60}$2.5$ & $5.2$ & $5.5$ & $4.8$ & $2.8$ & $2.6$ & $3.0$ & $6.1$ & $3.6$ & $4.7$ & $2.8$ & \cellcolor{red!100}$2.3$ \\
    H. Mean Rank (Max-F1) & $2.8$ & $4.3$ & $4.5$ & $5.0$ & \cellcolor{red!60}$2.5$ & $3.9$ & $3.0$ & $5.4$ & $2.8$ & $3.7$ & $6.5$ & \cellcolor{red!100}$1.8$ \\ \hline
\end{tabular}
}
\vspace{-4mm}
\end{center}
\end{table}

Tab.~\ref{tab:correctness} reports the harmonic mean of the ranking positions of each method over all datasets.
Besides AUROC, we consider additional metrics: Average Precision (AP) and Max-F1.
\method outperforms \textit{every} competitor in \textit{all} three metrics.

Overall, \method is the best option.
It wins in the vector datasets with known nonsingleton mcs; see the many red squares in the `Microclusters' section of Fig.~\ref{fig:correctness}. 
And, it ties with the competitors in the other vector datasets.
\method is also the only method directly applicable to metric data.
Every competitor is either nonapplicable or needs modifications when the data has no dimensions; see the red and orange rectangles at the top Fig.~\ref{fig:correctness}.
Choosing \method over the others is thus advantageous in \textit{both} vector and metric data.

\input{TAB/t-test}

\vspace{-2mm}
\subsection{Q2. \method is \principled}
\vspace{-1.5mm}

Tab.~\ref{tab:t-test} reports the results of an experiment performed to verify if the methods obey the axioms of Sec.~\ref{sec:axioms}.
Except for \method and Gen2Out, no method provides a score per microcluster; thus, they all fail to obey the axioms, by design.
We compared \method and Gen2Out statistically, by conducting two-sample t-tests, testing for $50$ datasets per axiom and shape of the cluster of inliers -- thus, summing up to $300$ datasets -- if the score obtained for the green microcluster (see Fig.~\ref{fig:axioms}) is larger than that of the red microcluster, against the null hypothesis that they are indifferent.
Note that Gen2Out misses both axioms by failing to find the microclusters in~\textit{every} one of the $200$ datasets with a cross- or arc-shaped cluster of inliers.
Distinctly, \method does \textit{not} miss any microcluster nor axiom.
Hence, our method obeys all the axioms a microcluster detector should follow; every single competitor~fails.

\vspace{-1.5mm}
\subsection{Q3. \method is \scalable}
\vspace{-1mm}

Fig.~\ref{fig:scalability} has results on the scalability of \method.
We plot runtime vs. data size for random samples of Uniform and Diagonal, considering their $2$- to $50$-dimensional versions.
The lines reflect the slopes expected from Lemma~\ref{lem:timecomplexity}.
Estimations and measurements agree.
As expected, \method scales subquadratically in every single case, regardless of the embedding dimension of the data. 
Particularly, note in \ref{fig:scalability}(iii) and (iv) that Uniform have respectively \textit{\underline{fractal} dimension~$\mathit{20}$ and $\mathit{50}$!}

Tab.~\ref{tab:runtime} reports runtime for \method and the other microcluster detectors in data of large cardinality or dimensionality.
Note that \method is the fastest method in nearly all cases, e.g., $\mathit{>50}$ \textit{times faster} than D.MCA in large data. 
We also emphasize our method is the only one that reports principled results because efficiency is worthless without effectiveness.

\begin{table}
\begin{center}
\caption{\textbf{Runtime Evaluation}\label{tab:runtime}
\vspace{-2.5mm}
}
\resizebox{0.7\columnwidth}{!}{
\begin{tabular}{l|c|c||c|}
       \diagbox{\textbf{Dataset}}{\textbf{Method}}
       & \underline{D.MCA}
       & \underline{Gen2Out}
       & \textbf{\underline{\method}} \\
\hline
    Gauss., Cross, Arc (\iaxiomshort) & $>10$ \textit{hours} & \cellcolor{red!60}$2$ \textit{hours} & \cellcolor{red!100}$12$ min. \\
    Gauss., Cross, Arc (\caxiomshort) & $>10$ \textit{hours} & \cellcolor{red!60}$2$ \textit{hours} & \cellcolor{red!100}$12$ min. \\
    HTTP & \cellcolor{red!60}$6$ min. & $18$ min. & \cellcolor{red!100}$4$ min. \\
    Satellite & \cellcolor{red!60}$25$ sec. & $59$ sec. & \cellcolor{red!100}$7$ sec. \\
    Speech & \cellcolor{red!100}$19$ sec. & $46$ sec. & \cellcolor{red!60}$34$ sec. \\ \hline\hline
    \textbf{\principled} & & & \checkmarkmethod \\ \hline
\end{tabular}
}
\vspace{-4mm}
\end{center}
\end{table}

\vspace{-2mm}
\subsection{Q4. \method is \practical} \label{sec:practical}
\vspace{-2mm}

\begin{figure*}[t]
    \begin{center}
        \begin{tabular}{c}
           \includegraphics[width=0.6\textwidth]{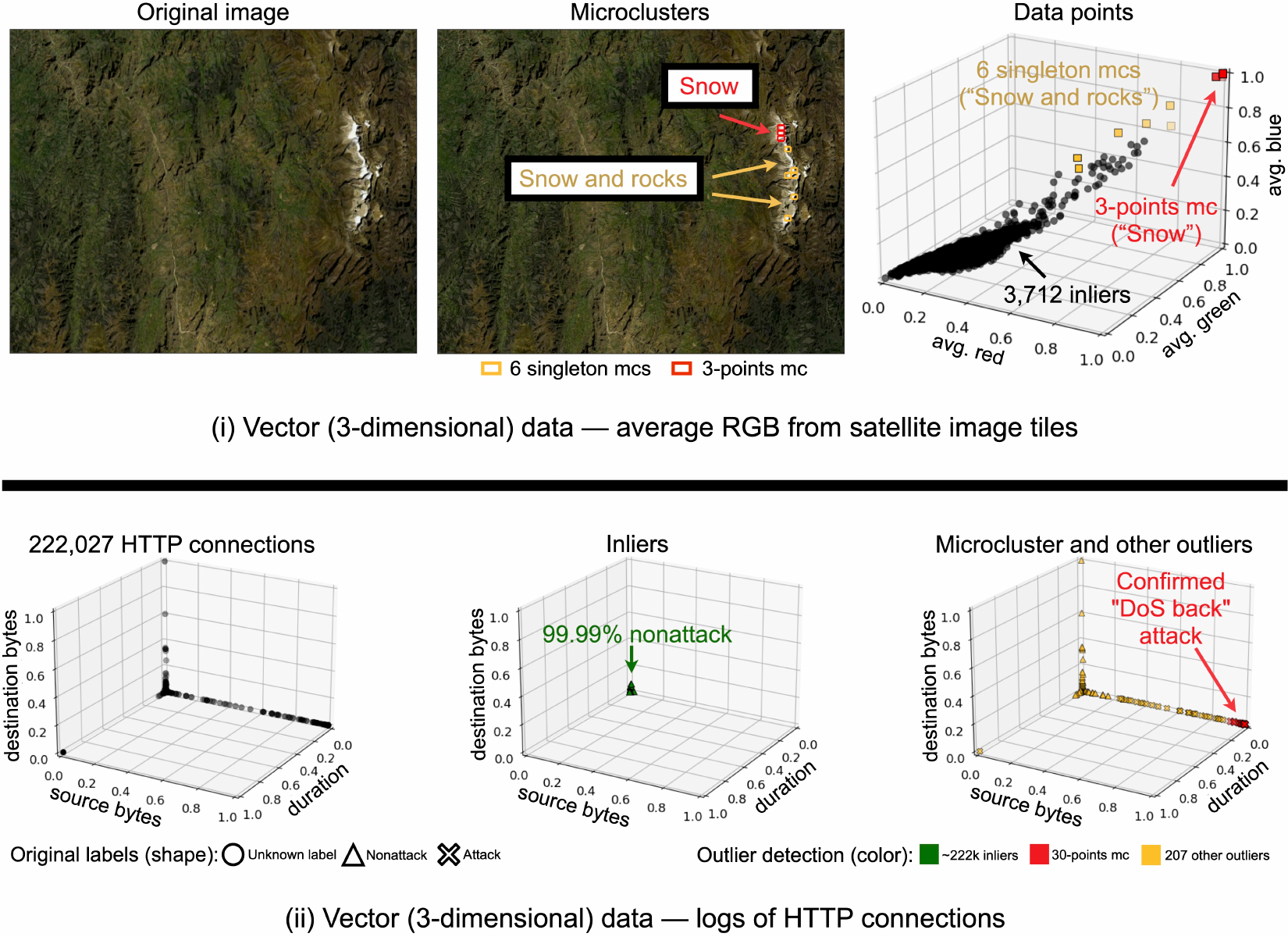} \\
        \end{tabular}
        \vspace{-2mm}
        \caption{\textbf{Q4. \method is \practical:} finding valid mcs unsupervisedly.
        (i)~on a satellite image -- it spots a $3$-points mc of snow in a volcano, and other~outlying tiles with snow and rocks;
        (ii) on network security data -- it spots a $30$-points mc of confirmed attacks, and other confirmed attacks.
        (best viewed in color)
        \label{fig:microclusters-others} 
        }
    \end{center}
    \vspace{-8mm}
\end{figure*}

\paragraph*{\hspace{-5mm} \underline{Attention routing}}
We studied the images on the left sides of Figs.~\ref{fig:crownJewel}(i) and~\ref{fig:microclusters-others}(i).
Each image was split into rectangular tiles from which average RGB values were extracted, thus leading to Shanghai and Volcanoes.
Ground truth labels are unknown and, thus, AUROC cannot be computed.
Regarding Shanghai,
\method found two $2$-points clusters formed from unusually colored roofs of buildings (red and blue tiles in Fig.~\ref{fig:crownJewel}(i) -- center), and other outlying tiles~(in yellow).
The red tiles are unusual and alike, and the same happens with the blue ones, while the yellow tiles are unusual but very distinct from one another.
The plot in the right side of Fig.~\ref{fig:crownJewel}(i) corroborates our findings; see the red and blue mcs, and the scattered, yellow outliers.
Similar results were seen in~Volcanoes, with a $3$-points cluster of snow found on the summit of the volcano; see Fig.~\ref{fig:microclusters-others}(i).
Thus, \method can successfully and unsuper-visedly route people's attention to notable image~regions.

\paragraph*{\hspace{-5mm} \underline{Unusual names}} Fig.~\ref{fig:crownJewel}(ii) reports results on nondimensional data.
We studied Last Names using the L-Edit distance.
\method earned a $0.75$ AUROC by finding the outliers on the left side of the figure. 
We investigated these names and discovered that they have a large variety of geographic origins; see the country flags.
Distinctly, the low-scored names mostly come from the UK; 
see the five ones with the lowest scores on the illustration's right side.
We conclude that \method distinguished English and NonEnglish names in the data.

\paragraph*{\hspace{-5mm} \underline{Unusual skeletons}} Fig.~\ref{fig:crownJewel}(iii) also regards nondimensional data.
We studied the $203$ graphs in Skeletons using the Graph edit distance.
\method earned a \textit{perfect} AUROC of $1$ by~finding all $3$ wild-animal skeletons on the figure's left side.
Thus, it successfully found the unusual, non-human skeletons.

\paragraph*{\hspace{-5mm} \underline{Network attacks}}
Fig.~\ref{fig:microclusters-others}(ii) reports results from HTTP.
The raw data is in the left-side plot;
there are $222$k connections described by numbers of bytes sent and received, and durations.
The inliers and outliers found by \method are at the center- and the right-side plots, respectively.
AUROC is $0.96$.
Note that $99.99\%$ of the inliers are not attacks; the outliers are either confirmed attacks or connections with a clear rarity, as they have oddly large durations, or numbers of bytes sent or received.
The most notable result is the detection of a $30$-points mc of \textit{confirmed} `DoS back' attacks,
which are characterized by sending too many bytes to a server aimed at overloading it.
Hence, our \method unsupervisedly found a cluster of frauds exploiting the same vulnerability in cybersecurity.

\vspace{-1mm}
\subsection{Q5. \method is \handsoff} \label{sec:sensitivity}
\vspace{-1mm}

\method needs only a few hyperparameters, namely, $\numberradii$, $\maxslope$, $\maxsize$.
It turns out that the values we have used 
($\numberradiidefaultonlyvalue$, $\maxslopedefaultonlyvalue$, $\maxsizedefaultonlyvalue$, respectively), are at a smooth plateau (see Fig.~\ref{fig:sensitivity}): that is, the accuracy is insensitive
to the exact choice of hyperparameter values.
Specifically, Fig.~\ref{fig:sensitivity} shows accuracy~vs.~$\numberradii$, $\maxslope$, $\maxsize$, respectively,
and every line corresponds to one of the datasets. Notice that all lines are near flat, highlighting the fact that \method needs no hyperparameter fine-tuning.
To avoid clutter, we only show the largest real dataset (HTTP) with line-point format.

\begin{figure}[htb]
    \centering
    \begin{tabular}{c}
       \includegraphics[width=0.875\linewidth]{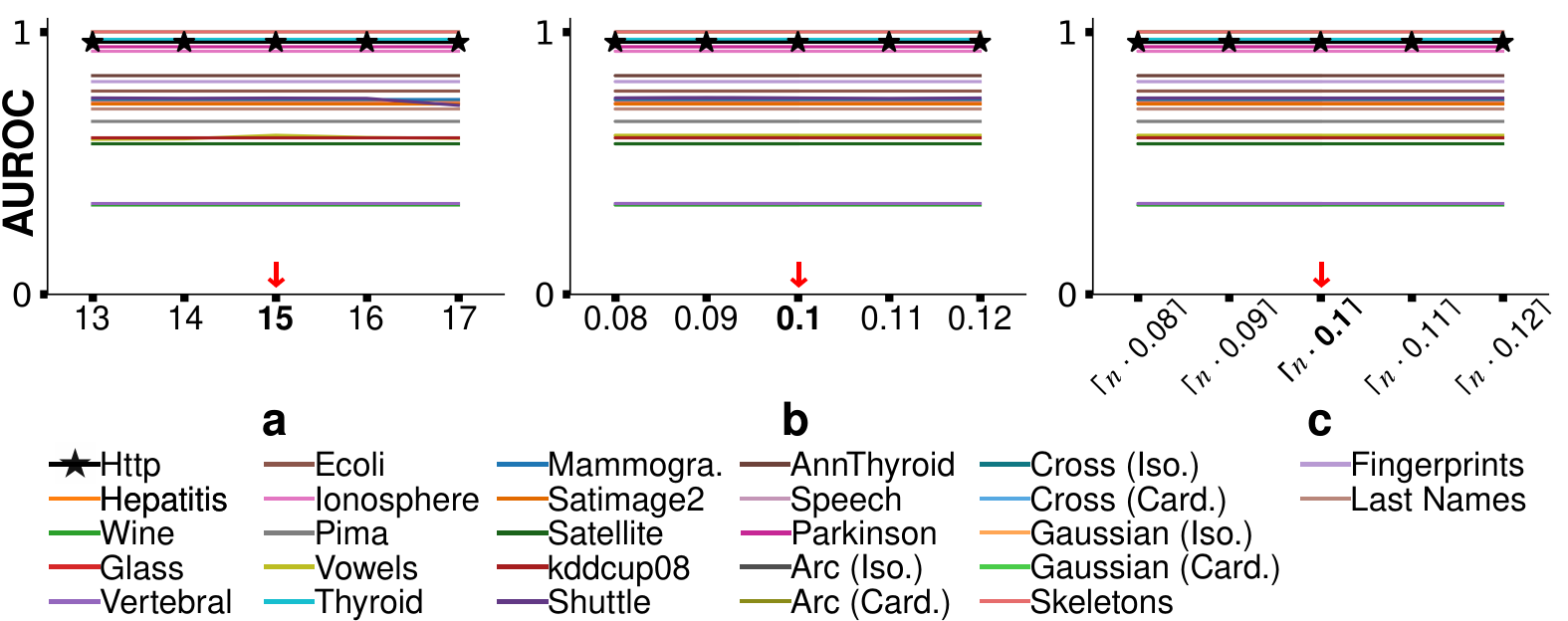} \\
    \end{tabular}
    \vspace{-2.5mm}
    \caption{\textbf{Q5. \method is \handsoff} and insensitive
    to the ($a,b,c$) hyper-parameter values:
    Accuracy has a smooth plateau.
    (best viewed in~color)
    \label{fig:sensitivity} 
    \vspace{-2mm}
    }
\end{figure}

%% file: TAB/t-test.tex
\begin{table}[t]
\centering
\caption{\textbf{Q2. \method is \principled:} it obeys all axioms. Every competitor misses one or more axioms. (best viewed in color)\label{tab:t-test}
\vspace{-2mm}
}
\resizebox{\columnwidth}{!}{
\begin{tabular}{l||cc|cc|cc||cc|cc|cc||}
      & \multicolumn{6}{c||}{\iaxiom}
      & \multicolumn{6}{c||}{\caxiom}  \\ 
      & \multicolumn{2}{c|}{Gaussian} & \multicolumn{2}{c|}{Cross} & \multicolumn{2}{c||}{Arc}
      & \multicolumn{2}{c|}{Gaussian} & \multicolumn{2}{c|}{Cross} & \multicolumn{2}{c||}{Arc}  \\
      & Stat & p-value & Stat & p-value & Stat & p-value & Stat & p-value & Stat & p-value & Stat & p-value \\
\hline\hline
\underline{\textbf{\method}} & \cellcolor{red!100}$2.6$ & \cellcolor{red!100}$9.3\mathrm{e}{-3}$ & \cellcolor{red!100}$9.0$ & \cellcolor{red!100}$1.6\mathrm{e}{-14}$ & \cellcolor{red!100}$22.3$ & \cellcolor{red!100}$3.1\mathrm{e}{-40}$ &
\cellcolor{red!100}$32.7$ & \cellcolor{red!100}$1.3\mathrm{e}{-54}$ & \cellcolor{red!100}$225.7$ & \cellcolor{red!100}$6.0\mathrm{e}{-135}$ & \cellcolor{red!100}$1,153.8$ & \cellcolor{red!100}$2.4\mathrm{e}{-204}$ \\
\hline\hline
\underline{Gen2Out} & \cellcolor{red!100}$5.1$ & \cellcolor{red!100}$1.2\mathrm{e}{-6}$ & \cellcolor{teal!80}{Fail} & \cellcolor{teal!80}{Fail} & \cellcolor{teal!80}{Fail} & \cellcolor{teal!80}{Fail} & \cellcolor{red!100}$32.1$ & \cellcolor{red!100}$6.6\mathrm{e}{-54}$ & \cellcolor{teal!80}{Fail} & \cellcolor{teal!80}{Fail} & \cellcolor{teal!80}{Fail} & \cellcolor{teal!80}{Fail}\\
ABOD & \cellcolor{teal!25}{N.A.} & \cellcolor{teal!25}{N.A.} & \cellcolor{teal!25}{N.A.} & \cellcolor{teal!25}{N.A.} & \cellcolor{teal!25}{N.A.} & \cellcolor{teal!25}{N.A.} & \cellcolor{teal!25}{N.A.} & \cellcolor{teal!25}{N.A.} & \cellcolor{teal!25}{N.A.} & \cellcolor{teal!25}{N.A.} & \cellcolor{teal!25}{N.A.} & \cellcolor{teal!25}{N.A.} \\
ALOCI & \cellcolor{teal!25}{N.A.} & \cellcolor{teal!25}{N.A.} & \cellcolor{teal!25}{N.A.} & \cellcolor{teal!25}{N.A.} & \cellcolor{teal!25}{N.A.} & \cellcolor{teal!25}{N.A.} & \cellcolor{teal!25}{N.A.} & \cellcolor{teal!25}{N.A.} & \cellcolor{teal!25}{N.A.} & \cellcolor{teal!25}{N.A.} & \cellcolor{teal!25}{N.A.} & \cellcolor{teal!25}{N.A.} \\
DB-Out & \cellcolor{teal!25}{N.A.} & \cellcolor{teal!25}{N.A.} & \cellcolor{teal!25}{N.A.} & \cellcolor{teal!25}{N.A.} & \cellcolor{teal!25}{N.A.} & \cellcolor{teal!25}{N.A.} & \cellcolor{teal!25}{N.A.} & \cellcolor{teal!25}{N.A.} & \cellcolor{teal!25}{N.A.} & \cellcolor{teal!25}{N.A.} & \cellcolor{teal!25}{N.A.} & \cellcolor{teal!25}{N.A.}\\
\underline{D.MCA} & \cellcolor{teal!25}{N.A.} & \cellcolor{teal!25}{N.A.} & \cellcolor{teal!25}{N.A.} & \cellcolor{teal!25}{N.A.} & \cellcolor{teal!25}{N.A.} & \cellcolor{teal!25}{N.A.} & \cellcolor{teal!25}{N.A.} & \cellcolor{teal!25}{N.A.} & \cellcolor{teal!25}{N.A.} & \cellcolor{teal!25}{N.A.} & \cellcolor{teal!25}{N.A.} & \cellcolor{teal!25}{N.A.} \\
FastABOD & \cellcolor{teal!25}{N.A.} & \cellcolor{teal!25}{N.A.} & \cellcolor{teal!25}{N.A.} & \cellcolor{teal!25}{N.A.} & \cellcolor{teal!25}{N.A.} & \cellcolor{teal!25}{N.A.} & \cellcolor{teal!25}{N.A.} & \cellcolor{teal!25}{N.A.} & \cellcolor{teal!25}{N.A.} & \cellcolor{teal!25}{N.A.} & \cellcolor{teal!25}{N.A.} & \cellcolor{teal!25}{N.A.} \\
iForest & \cellcolor{teal!25}{N.A.} & \cellcolor{teal!25}{N.A.} & \cellcolor{teal!25}{N.A.} & \cellcolor{teal!25}{N.A.} & \cellcolor{teal!25}{N.A.} & \cellcolor{teal!25}{N.A.} & \cellcolor{teal!25}{N.A.} & \cellcolor{teal!25}{N.A.} & \cellcolor{teal!25}{N.A.} & \cellcolor{teal!25}{N.A.} & \cellcolor{teal!25}{N.A.} & \cellcolor{teal!25}{N.A.} \\
LOCI & \cellcolor{teal!25}{N.A.} & \cellcolor{teal!25}{N.A.} & \cellcolor{teal!25}{N.A.} & \cellcolor{teal!25}{N.A.} & \cellcolor{teal!25}{N.A.} & \cellcolor{teal!25}{N.A.} & \cellcolor{teal!25}{N.A.} & \cellcolor{teal!25}{N.A.} & \cellcolor{teal!25}{N.A.} & \cellcolor{teal!25}{N.A.} & \cellcolor{teal!25}{N.A.} & \cellcolor{teal!25}{N.A.} \\
LOF & \cellcolor{teal!25}{N.A.} & \cellcolor{teal!25}{N.A.} & \cellcolor{teal!25}{N.A.} & \cellcolor{teal!25}{N.A.} & \cellcolor{teal!25}{N.A.} & \cellcolor{teal!25}{N.A.} & \cellcolor{teal!25}{N.A.} & \cellcolor{teal!25}{N.A.} & \cellcolor{teal!25}{N.A.} & \cellcolor{teal!25}{N.A.} & \cellcolor{teal!25}{N.A.} & \cellcolor{teal!25}{N.A.} \\
ODIN & \cellcolor{teal!25}{N.A.} & \cellcolor{teal!25}{N.A.} & \cellcolor{teal!25}{N.A.} & \cellcolor{teal!25}{N.A.} & \cellcolor{teal!25}{N.A.} & \cellcolor{teal!25}{N.A.} & \cellcolor{teal!25}{N.A.} & \cellcolor{teal!25}{N.A.} & \cellcolor{teal!25}{N.A.} & \cellcolor{teal!25}{N.A.} & \cellcolor{teal!25}{N.A.} & \cellcolor{teal!25}{N.A.} \\
\underline{RDA} & \cellcolor{teal!25}{N.A.} & \cellcolor{teal!25}{N.A.} & \cellcolor{teal!25}{N.A.} & \cellcolor{teal!25}{N.A.} & \cellcolor{teal!25}{N.A.} & \cellcolor{teal!25}{N.A.} & \cellcolor{teal!25}{N.A.} & \cellcolor{teal!25}{N.A.} & \cellcolor{teal!25}{N.A.} & \cellcolor{teal!25}{N.A.} & \cellcolor{teal!25}{N.A.} & \cellcolor{teal!25}{N.A.} \\
\hline
\end{tabular}
}
\end{table}

%% file: 060conclusion.tex
\vspace{-3mm}
We presented \method to address the microcluster-detection problem.
The main idea is to leverage our proposed ‘Oracle’ plot (\fplengthname vs. \mplengthname). \method achieves five goals:
\begin{compactenum}[{G}1.]
\item \textbf{\generalinput:} \method works with any metric dataset, including nondimensional ones, as shown in Fig.~\ref{fig:crownJewel}.
    It is achieved by depending solely on distances.
\item \textbf{\generaloutput:} \method ranks singleton (`one-off' outliers) and nonsingleton mcs \textit{together}, by anomalousness. See Probl.~\ref{def:problem} and Def.~\ref{def:score}.
    It is achieved thanks to
    a new compression-based idea (Fig.~\ref{fig:score}) to compute scores.
\item \textbf{\principled:} \method obeys axioms;
see Tab.~\ref{tab:t-test}.
It is achieved thanks to the new group axioms of Fig.~\ref{fig:axioms}~and our score-computation strategy that match human~intuition.
\item \textbf{\scalable:} \method is subquadratic on the number of points, as shown in Lemma~\ref{lem:timecomplexity} and Fig.~\ref{fig:scalability}. 
It is made possible by carefully building on spatial joins and metric~trees.
\item \textbf{\handsoff:} \method needs no manual tuning. 
It~is achieved due to our MDL-based idea to get the \mindistname~$\mindist$ from the given data; see Def.~\ref{def:mindist} and Figs.~\ref{fig:minimum:cost} and \ref{fig:sensitivity}.~We also set hyperparameters to the reasonable defaults of~Alg.~\ref{algo:main}.
\end{compactenum}

\textit{\underline{No}} competitor fulfills all of these goals; see Tab.~\ref{tab:salesman}. 
Also, \method is deterministic, ranks the points, and gives explainable results.
We studied $31$ real and synthetic datasets~and showed \method outperforms $11$ competitors, especially when the data has nonsingleton microclusters or is nondimensional.
We also showcased \method's ability to find meaningful mcs in graphs, fingerprints, logs of network connections, text data, and satellite images.
For example, it found a $30$-points mc of \textit{confirmed} attacks in the network logs,
taking \textit{$\mathit{\sim\hspace{-0.5mm}3}$ minutes} for $222$K points on a stock desktop; 
see~Fig.~\ref{fig:microclusters-others}(ii).

{\bf Reproducibility:} For reproducibility, our data and code are available at \repository.